\documentclass[conference]{IEEEtran}
\usepackage{times}

\usepackage{multicol}
\usepackage[bookmarks=true]{hyperref}

\usepackage{amssymb}		
\usepackage{graphicx}		
\usepackage{hyperref}		
\usepackage{amsmath}
\usepackage{amsfonts}
\usepackage{amsthm}
\usepackage[linesnumbered,ruled,noend]{algorithm2e}
\usepackage{algpseudocode}
\usepackage{xcolor}
\usepackage{cite}
\usepackage{xspace} 
\usepackage{comment}
\usepackage[utf8]{inputenc}
\usepackage{multirow}
\usepackage{wrapfig}
\usepackage{ifthen}
\usepackage{eqparbox}
\usepackage{svg}
\usepackage{subcaption}
\usepackage{float}
\usepackage{pdflscape}
\usepackage{rotating}
\usepackage{eso-pic}

\begin{document}
\AddToShipoutPictureBG*{%
  \AtPageUpperLeft{%
    \hspace{18cm}%
    \raisebox{-1.5cm}{%
      \makebox[0pt][r]{Submitted to the International Conference on Intelligent Robots and Systems (IROS), October 2025.}}}}

\title{Extended Version: Multi-Robot Motion Planning with Cooperative Localization}

\author{Anne Theurkauf, Nisar Ahmed, Morteza Lahijanian}

\maketitle

\newtheorem{theorem}{Theorem}
\newtheorem{problem}{Problem}
\newtheorem{myexam}{Example}
\newtheorem{assumption}{Assumption}
\newtheorem{proposition}{Proposition}
\newtheorem{corollary}{Corollary}
\newtheorem{lemma}{Lemma}
\newtheorem{definition}{Definition}
\newtheorem{remark}{Remark}

\newcommand{\proj}[1]{\textsc{proj}_{#1}}
\newcommand{\A}{\mathcal{A}}
\newcommand{\B}{\mathcal{B}}
\newcommand{\C}{\mathcal{C}}
\newcommand{\E}{\mathcal{E}}
\newcommand{\N}{\mathcal{N}}
\newcommand{\R}{\mathcal{R}}
\newcommand{\Scal}{\mathcal{S}}
\newcommand{\U}{\mathcal{U}}
\newcommand{\W}{\mathcal{W}}
\newcommand{\X}{\mathcal{X}}
\newcommand{\Y}{\mathcal{Y}}

\newcommand{\xeuc}{\text{x}}

\newcommand{\safe}{\text{safe}}
\newcommand{\coll}{\text{coll}}

\newcommand{\expBelief}{\textbf{b}}

\newcommand{\ml}[1]{\textcolor{blue}{[ML: #1]}}
\newcommand{\at}[1]{\textcolor{purple}{[AT: #1]}}
\newcommand{\nra}[1]{\textcolor{orange}{[NRA: #1]}}
\newcommand{\qh}[1]{\textcolor{blue}{[QH: #1]}}

\begin{abstract}

We consider the uncertain multi-robot motion planning (MRMP) problem with cooperative localization (CL-MRMP), under both motion and measurement noise,
where each robot can act as a sensor for its nearby teammates.
We formalize CL-MRMP as a chance-constrained motion planning problem, and 
propose a safety-guaranteed algorithm that explicitly accounts for robot-robot correlations. Our approach extends a sampling-based planner to solve CL-MRMP while preserving probabilistic completeness. To improve efficiency, we introduce novel biasing techniques. We evaluate our method across diverse benchmarks, demonstrating its effectiveness in generating motion plans, with significant performance gains from biasing strategies.

\end{abstract}

\IEEEpeerreviewmaketitle

\section{Introduction}

Multi-robot teams are powerful assets, 
offering diverse capabilities and enabling parallel operation to improve efficiency and coverage in applications ranging from warehouse automation to space exploration \cite{Pomares2023, Leitner2009}. 
Multi-robot teams are particularly advantageous 
in adversarial environments such as GPS-denied settings, as each robot can serve as a sensor for others, reducing overall uncertainty through cooperation localization (CL). For example, robots can obtain relative measurements from nearby teammates \cite{Gao2020} to correct drift from inertial sensors \cite{Russell2020_CLforUAV, Mokhtarzadeh2014_CoopInertialNav}
(e.g., Fig.~\ref{fig:example traj}). 
This however introduces a significant challenge in motion planning,
as the planner must not only determine collision-free trajectories but also coordinate CL opportunities.  
In this work, we focus on the multi-robot motion planning (MRMP) problem for teams operating with CL, which we refer to as \textit{CL-MRMP}.


A general approach to MRMP is centralized, coupled planning,  
which models all robots as a single meta-agent and plans for them simultaneously in the joint state space \cite{Wagner2025_MRMP}.
This is computationally challenging 
since the search space grows exponentially with the number of robots. 
Nevertheless, it allows the planner to maintain information about all the robots.
More scalable 
decoupled methods
\cite{Kotting2022_KCBS, Theurkauf2024_CCKCBS} 
plan for individual robots and selectively resolve conflicts. 
Yet, they neglect or only partly account for robot-robot interactions. This is problematic for scenarios requiring safe navigation via CL, 
which
introduces correlations between robots' state estimates~\cite{Roumeliotis2000} that must be considered for feasible planning. 

Online methods can mitigate this challenge by planning for the robot team over a short horizon. A common approach for CL is online distributed planning, where each robot plans locally over a short horizon while accounting for nearby robots and obstacles \cite{Patwardhan2023_DistMRMP,VanParys2016_OnlineDistMultiVehicle}. Another widely used technique is formation control, in which robot formations are designed to minimize uncertainty for the team \cite{Hidaka2005, Qu2021}, and online execution consists of maintaining the prescribed formation \cite{Guillet2014, Alonso-Mora2019}. While these methods are effective in unknown environments, they lack formal guarantees for safety or completeness.

\begin{figure}
    \centering
    \begin{subfigure}[b]{0.21\textwidth}
        \centering
        \fontsize{5}{5}\selectfont 
        \includegraphics[width=\textwidth,trim={3.8cm 1.9cm 3.9cm 2.1cm}, clip]{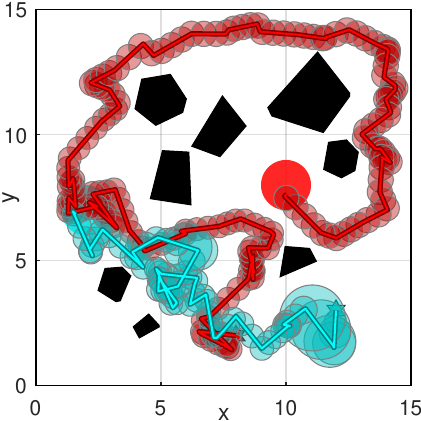}
        \caption{\small Trajectories}
        \label{fig:Random trajectory}
    \end{subfigure}
    \begin{subfigure}[b]{0.23\textwidth}
        \centering
        \fontsize{5}{5}\selectfont 
        \includegraphics[width=\textwidth,trim={4cm 2.1cm 4cm 3cm}, clip]{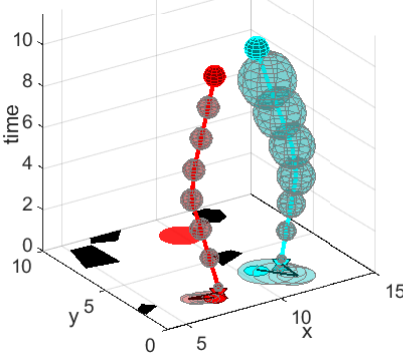}
        \caption{\small Trajectory with time axis}
        \label{fig:Random trajectory zoom}
    \end{subfigure}
    \caption{\small
    CL-MRMP solution plan for 2 robots with motion and sensing uncertainties 
    (initial states are
    near the bottom of the figure, and their goal regions are indicated by red and cyan circles). 
    Cyan robot lacks 
    onboard sensors, 
    but the solution plan enables it to use the red robot as a sensor, reducing its uncertainty and allowing it to successfully navigate to its goal region. Afterward, the plan guides the red robot to its goal. 
    (trajectory circles: 2$\sigma$ bounds). 
    }
    \label{fig:example traj}
    \vspace{-5mm}
\end{figure}



This work formally defines the CL-MRMP problem and proposes a safety-guaranteed algorithm that explicitly accounts for both motion and measurement uncertainty. We first establish the necessity of a centralized estimator to track the coupling of robot states induced by CL. Then, we formulate CL-MRMP as a chance-constrained planning problem, allowing us to adapt existing algorithms. We extend the sampling-based algorithm in \cite{Ho2022_GBT} to solve CL-MRMP and demonstrate that it inherits the probabilistic completeness properties of the underlying algorithm. Additionally, we introduce biasing techniques to improve performance. Our algorithm is evaluated on a diverse set of benchmarking problems, and the results show that it effectively addresses the CL-MRMP problem, with our biasing techniques significantly enhancing performance.  

Our main contributions are: (i) a formalization of the CL-MRMP problem as a chance-constrained motion planning problem, (ii) a sampling-based planning algorithm that accurately accounts for robot-robot correlations in state estimates, (iii) novel biasing techniques for more efficient planning with CL, and (iv) extensive benchmarks and illustrative examples demonstrating the efficacy of our formulation and approach.

\section{Problem Formulation}
\label{sec:Problem}
In this work, we consider $N_A\in \mathbb{N}_{\geq 2}$ robots
that evolve in a shared workspace $\W\subset\mathbb{R}^w$, $w\in \{2,3\}$ under both motion and sensor uncertainties.
These robots are capable of sensing and communicating with nearby robots.

\subsection{Robot Dynamics}
  The evolution of robot $i \in \{1,\ldots, N_A\}$ is governed by stochastic linear dynamics
\begin{equation}
\label{eq:dynamics_robot}
    x_{k+1}^i = A^i x_k^i + B^i u_k^i + w_k^i,
\end{equation}
where $x_k^i \in \X^i\subseteq \mathbb{R}^{n_i}$ and $u_k^i \in \U^i\subseteq  \mathbb{R}^{m_i}$ are the state and control at time step $k \in \mathbb{N}_{\geq 0}$, respectively with associated matrices $A^i\in \mathbb{R}^{n_i\times n_i}$ and $B^i\in \mathbb{R}^{n_i\times m_i}$, and $w_k^i \in \mathbb{R}^{n_i}$ is Gaussian distributed noise with zero mean and $Q^i\in \mathbb{R}^{n_i\times n_i}$ covariance, i.e., $w_k^i\sim \N(0,Q^i)$.


Each robot body $\B^i$ is defined as a set of points. We use $\B^i(x_k^i)\subset\W$
to denote the set of points in the workspace that robot $i$ occupies when placed at state $x_k^i$. 

\subsection{Robot Measurements}

We assume the robots are equipped with \emph{proprioceptive} and \emph{exteroceptive} sensors, allowing them to measure (and communicate) not only their own states but also those of nearby robots. Proprioceptive measurements pertain to an individual robot's state and are independent of all others (e.g., velocity encoder). 
Such measurements of robot $i$ are given by:
\begin{equation}
\label{eq:proprioceptive measurement}
    y_k^{i,prop} = C^{i,prop} x_k^i + v_k^{i,prop},
\end{equation}
where $y_k^{i,prop}\in \mathbb{R}^{q^{prop}_i}$, 
with associated matrix $C^{i,prop}\in\mathbb{R}^{q_i^{prop}\times n_i}$, and zero-mean Gaussian distributed noise $v_k^{i,prop}\sim \N(0,R^{i,prop})$, with covariance $R^{i,prop}\in \mathbb{R}^{q^{prop}_i\times q^{prop}_i}$.

In contrast, exteroceptive measurements are taken relative to other robots, introducing dependencies between their states (e.g., range). 
Such measurement between robots $i$ and $j$, where $i\neq j \in \{1,\ldots,N_r\}$, are modeled as
\begin{equation}
\label{eq:exteroceptive measurement}
    y_k^{ij,ext}=C_k^{ij,ext} x_k^{ij} + v_k^{ij,ext},
\end{equation}
where $y_k^{ij,ext}\in \mathbb{R}^{m^{ext}_{ij}}$ is 
defined with respect to the composed state $x_k^{ij}=[x_k^i, x_k^j]^T\in \mathbb{R}^{n_i+n_j}$, with associated mapping $C_k^{ij,ext}\in\mathbb{R}^{m^{ext}_{ij} \times (n_i+n_j)}$ and noise $v_k^{ij,ext}\sim \N(0,R_k^{ij,ext})$ and covariance $R_k^{ij,ext}\in \mathbb{R}^{m^{ext}_{ij}\times m^{ext}_{ij}}$. 
%
Measurement $y_k^{ij,ext}$
is \emph{only enabled} when the robots are within workspace radius $r_{ext}\in\mathbb{R}_{\geq 0}$, i.e., for $\proj{\W}: \cup_{i=1}^{N_A} \X^i \to \W$ denoting the projection operator of the state into the workspace, 
\begin{multline}
\label{eq:exteroceptive radius condition}
        \| \proj{\W}(x_k^i) - \proj{\W}(x_k^j) \| \leq r_{ext} \implies \\
        y_k^{ij,ext} \text{  available to robots $i$ and $j$},
\end{multline}
else $y_k^{ij,ext}$ does not exist.  
In contrast, note $y^{i,prop}_k$ is always available to robot $i$. 
Also note that $C_k^{ij,ext}$ is time-varying as 
condition \eqref{eq:exteroceptive radius condition} depends on the time-varying robot states. 

\subsection{Centralized Estimation and Cooperative Localization}
\label{sec:estimation and control}
Exteroceptive measurements induce correlations between robot states. To fully account for all correlations, we assume a centralized estimator 
over the composed states of all robots that has access to all measurements. 
We denote concatenation of a set of $N$ (column) vectors $\{\mathrm{v}_i\}_{i=1}^N$ 
as 
\begin{equation}
\textsc{Concat}(\{\mathrm{v}_i\}_{i=1}^N) = 
    [\mathrm{v}_1^T, \mathrm{v}_2^T, \ldots, \mathrm{v}_N^T]^T,
\end{equation}
and the block diagonal matrix constructed from a set of $N$ matrices $\{M_i\}_{i=1}^N$ as $\textsc{BlockDiag}(\{M_i\}_{i=1}^N)$.

The composed dynamics of the team of robots are given by: 
\begin{equation}
    \label{eq:dynamics_composed}
    X_{k+1} = A X_k + B U_k + W_k.
\end{equation}
where $X_k=\textsc{Concat}(\{x_k^i\}_{i=1}^{N_A})\in \X\subseteq\mathbb{R}^{n_N}$ with $n_N = \sum_{i=1}^{N_A} n_i$ and $U_k=\textsc{Concat}(\{u_k^i\}_{i=1}^{N_A}) \in \U \subseteq\mathbb{R}^{m_N}$ with $m_N = \sum_i^{N_A} m_i$ are the composed state and control, respectively, and 
matrices $A=\textsc{BlockDiag}(\{A^i\}_{i=1}^{N_A})$ and $B=\textsc{BlockDiag}(\{B^i\}_{i=1}^{N_A})$. The composed noise is distributed as $W_k \sim \N(0,Q)$, with $Q=\textsc{BlockDiag}(\{Q^i\}_{i=1}^{N_A})$.


The composed robot proprioceptive measurement model is:
\begin{equation}
    Y_k^{prop} = C^{prop} X_k + V_k^{prop}
\end{equation}
where $Y_k^{prop}=\textsc{Concat}(\{y_k^{i,prop}\}_{i=1}^{N_A})$ and $V_k^{prop}\sim \N(0,R^{prop})$. Because the individual proprioceptive measurements are independent, $C^{prop}=\textsc{BlockDiag}(\{ C^{i,prop}\}_{i=1}^{N_A})$ and $R^{prop}=\textsc{BlockDiag}(\{R^{i,prop}\}_{i=1}^{N_A})$. 

The composed exteroceptive measurement model is:
\vspace{-2mm}
\begin{equation}
\label{eq:exteroceptive composed measurement}
    Y_k^{ext} = C_k^{ext} X_k + V_k^{ext}
\end{equation}
where $Y_k^{ext}=\textsc{Concat}( \{y_k^{ij,ext}\}_{\models} )$, with $\{y_k^{ij,ext}\}_{\models}$ being the set of measurements obtained according to  \eqref{eq:exteroceptive radius condition}, and noise $V_k^{ext}\sim\N(0,R_k^{ext})$ with $R_k^{ext}=\textsc{BlockDiag}(\{R_k^{ij,ext}\}_{\models})$. Note that the constituent $C_k^{ij,ext}$ matrices are defined over the states of the two robots $i$ and $j$, and therefore $C_k^{ext}$ is not a simple $\textsc{BlockDiag}$. It instead must be constructed to preserve the mapping of the individual  $C_k^{ij,ext}$ to the full composed state. 
Additionally, while $C^{prop}$ is time-invariant, $C_k^{ext}$ is time-varying as the distances between robots change (as noted earlier). 
The full measurement equation is thus: 
\begin{equation}
\label{eq:measurement_composed}
    Y_k = C_k X_k + V_k,
\end{equation}
where $Y_k=\textsc{Concat}(\{Y_k^{prop}, Y_k^{ext}\})$, $C_k=\textsc{Concat}(\{C_k^{prop}, C_k^{ext}\})$, and noise $V_k \sim \N(0,R_k)$, with covariance $R_k=\textsc{BlockDiag}(\{R^{prop}, R_k^{ext}\})$. If no robot pairs satisfy \eqref{eq:exteroceptive radius condition} at time step $k$, then \eqref{eq:measurement_composed} reduces to $Y_k=Y_k^{prop} \!$ with $C_k=C^{prop}$.

The composed system therefore reduces the multi-robot system to a single linear system with Gaussian noise governed by \eqref{eq:dynamics_composed} and \eqref{eq:measurement_composed}. Hence, we can use a centralized Kalman Filter (KF) to maintain an online estimate of $X_k$ as Gaussian belief $b(X_k)$ with mean $\hat{X}_k \in \mathbb{R}^{n_N}$ and covariance $\Sigma_k \in \mathbb{R}^{n_N\times n_N}$, 
$$X_k \sim b(X_k) = \N(\hat{X}_k, \Sigma_k).$$
Note that $\Sigma_k$ fully captures robot-robot correlations. From this, we can extract the marginal belief for each robot.  We denote the belief of robot $i$ state $x_k^i$ by $b(x_k^i)$, i.e., $$x_k^i \sim b(x_k^i)=\mathcal{N}(\hat{x}_k^i,\Sigma_k^i).$$

\subsection{Motion Plan and Control}
\label{sec:Control}
We define a \textit{motion plan} for robot $i$ to be a tuple $(\check{u}^i, \check{x}^i, \check{C})$, where 
$\check{u}^i=(\check{u}_0^i,\check{u}_1^i,...,\check{u}_{T-1}^i) \in (\U^i)^*$  is a nominal control trajectory; 
$\check{x}^i=(\check{x}_0^i,\check{x}_1^i,...,\check{x}_T^i) \in (\X^i)^*$ is the nominal state trajectory obtained by propagation of the nominal dynamics  $\check{x}^i_{k+1} = A^i \check{x}^i_k + B^i \check{u}^i_k$ on $\check{u}^i$; 
and
$\check{C} = (\check{C}_0, \check{C}_1, \ldots, \check{C}_{T-1})$ is a sequence of (measurement) matrices used for KF. 
%

Robot $i$ executes the motion plan $(\check{u}^i, \check{x}^i, \check{C})$ with the feedback control law $    u_k^i=\check{u}_k^i - K^i(\hat{x}^i_k-\check{x}_k^i)$,
with gain matrix $K^i\in\mathbb{R}^{m_i\times n_i}$, where $\hat{x}^i_k$ is the centralized KF state estimate obtained using the measurement matrix $\check{C}_k$.  Note that 
(i) this controller stabilizes the robot about the nominal trajectory $\check{x}^i$, and 
(ii) the KF relies on exteroceptive measurements at time step $k$ if $\check{C}_k$ requires them (the following section elaborates on the feasibility of this). Finally, note that these definitions of the motion plan and controller extend those in~\cite{Theurkauf2024_CCKCBS} to appropriately account for CL.



\subsection{Probabilistic Objectives}
\label{sec:prob obj}

Each robot $ i $ is assigned a goal region $ \mathcal{X}^i_G \subset \mathcal{X}^i $ in its state space,
which
contains $ N_O $ (disjoint) static obstacles  $\X^i_{O_j} \subset \X^i$, where $j \in \{1,\ldots, N_O\}$. 
Each of the $ N_A $ robots acts as a dynamic obstacle and sensor. 
The motion planning task is to compute a plan for each robot to reach its goal while avoiding collisions with both static and moving obstacles. Additionally, if the motion plan requires two robots to use exteroceptive measurements for CL at time step $ k $, those robots must be within a distance $ r_{ext} $ of each other at that time step during the execution.  
Since the robots operate under uncertainty, all three requirements (goal satisfaction, collision avoidance, and CL) must be analyzed probabilistically. 

The probability of robot $ i $ being in 
$ \mathcal{X}_G^i $ at time 
$ k $ is 
$P_G^{i_k} = P(x_k^i \in \mathcal{X}_G^i) = \int_{\mathcal{X}^i_G} b(x^i_k)(s) \, ds,$  
where $ b(x^i_k)(s) $ 
is the probability density function evaluated at $ s $.  
Similarly, the probability of colliding with a static obstacle is  
$P^{i_k}_O = P(x^i_k \in \mathcal{X}^i_O) = \int_{\mathcal{X}^i_O} b(x^i_k)(s) \, ds,$  
where $ \mathcal{X}^i_O = \bigcup_{j=1}^{N_O} \mathcal{X}^i_{O_j} $ 
is the union of all obstacle regions.  
Finally, the probability of colliding with another robot $ j $ is 
$P^{ij_k}_{\text{coll}} = P((x^i_k, x^j_k) \in \mathcal{X}^{ij}_{\text{coll}}),$  
where  
$\mathcal{X}^{ij}_{\text{coll}} = \{(x^i_k, x^j_k) \in \mathcal{X}^i \times \mathcal{X}^j \mid \B^i(x^i_k) \cap \B^j(x^j_k) \neq \emptyset\}$  
is the set of states where the two robots' projected positions in the workspace overlap (collide). 

For probabilistic analysis of CL, 
let $ r^{ij}_k = \|\proj{\mathcal{W}}(x^i_k) - \proj{\mathcal{W}}(x^j_k) \| $ be the workspace distance between robots $ i $ and $ j $ at time 
$ k $. Since $ x^i_k $ and $ x^j_k $ are random variables, $ r^{ij}_k $ is also a random variable 
distributed as $r^{ij}_k\sim b(r^{ij}_k)$. The probability that robots $ i $ and $ j $ fail to use their exteroceptive measurements at time step $ k $ is 
\begin{align}
    \label{eq:prob CL}
    P(r^{ij}_k > r_{ext}) = 
    1 - \int_{0}^{r_{ext}} b(r^{ij}_k)(s) \, ds.
\end{align}  
If the motion plan requires CL,
and it is unavailable during execution, this results in a failure.
The probability of this failure must be captured in planning, as formalized below.

For two matrices $ D $ and $ E $, let $ D \sqsubset E $ denote that $ D $ is a submatrix of $ E $. Then, the probability of failure of CL for robots $ i $ and $ j $ at time step $ k $ under a given motion plan with measurement matrix $\check{C}_k$ is defined as  
\begin{equation}
\label{eq:CL probability}
    P^{ij_k}_{\neg CL} = 
    \begin{cases}
        P(r^{ij}_k > r_{ext}), & \text{if } C_k^{ij,ext} \sqsubset \check{C}_k, \\
        0, & \text{otherwise}.        
    \end{cases}
\end{equation}

Below, we state the safe CL-MRMP problem.

\subsection{CL-MRMP Problem}
\label{subsec:problem}
Consider $N_A$ robots with noisy dynamics in \eqref{eq:dynamics_robot} and noisy 
measurements in \eqref{eq:proprioceptive measurement} and \eqref{eq:exteroceptive measurement}, equipped with the (centralized) KF and feedback control law 
described in Secs.~\ref{sec:estimation and control} and \ref{sec:Control}.
Given a set of initial distributions $\{x_0^i = \N(\hat{x}^i_0, \Sigma^i_0)\}_{i=1}^{N_A}$, goal regions $\{\X^i_G\}_{i=1}^{N_A}$, and obstacles $\X_{O}$, 
and safety threshold $p_\text{safe}$, compute motion plan $(\check{u}^i, \check{x}^i, \check{C})$ for each robot $i\in \{1,\ldots,N_A\}$ from its initial state to goal region such that
\begin{subequations}
    \begin{align}   
       \label{eq:cc obs}
        &P^{i_k}_O + \!\!\!\!\! \sum_{j=1, j \neq i}^{N_A} \!\!\! (P^{ij_k}_{coll} + P_{\neg CL}^{ij_k})  \leq \! 1 \!- p_\text{safe} \;\; \forall k\in \{1,..,T\}, \\
        &P^{i_T}_G \geq p_\text{safe} 
       \label{eq:cc goal}
    \end{align}
    \label{eq:chance constraints}
\end{subequations}
where $P^{i_k}_O$, $P^{ij_k}_{coll}$, and $P_{\neg CL}^{ij_k}$ are defined in Sec.~\ref{sec:prob obj}.  The safety requirements in ~\eqref{eq:chance constraints} are known as \textit{chance constraints}.


\textbf{Approach Overview:} 
%
%
%
This problem is challenging as it requires both motion planning and scheduling cooperative measurements for a set of uncertain robots. As discussed earlier, 
this requires assuming a centralized estimator
to accurately track state correlations without losing information. 
Given this assumption, where the state of every robot is accessible, we adopt a centralized planning framework. Since the system is already represented as a single composed robot, a coupled planning approach is a natural first step toward effectively solving the CL-MRMP problem.  
Based on the lessons learned from this study, we can investigate a decoupled approach in future work.

Here, we propose a coupled planning framework that explicitly accounts for robot-robot correlations during planning. Our approach extends existing single-robot Gaussian belief planners by incorporating constraints on robot-robot collisions and exteroceptive measurement availability. 
We also introduce biasing techniques to enhance exploration of CL behaviors.

\section{Sampling-Based Planner Framework}
\label{sec:Planner Framework}
In this section, we detail a safety-guaranteed algorithm for CL-MRMP. Safety constraints are formulated with respect to the online belief $b(X^i_k)$ and thus conditioned on realized measurements. Without prior knowledge of the particular realization of $b(X^i_k)$, a motion plan cannot be guaranteed to satisfy the safety constraints. Our approach instead reasons over all possible online distributions by planning over the \emph{expected} belief, 
\begin{align*}
    \label{eq: expected belief}
    \expBelief(X_k) \! = \mathbb{E}_Y[b(X_k | X_0, Y_{0:k})] 
   = \!\!\int_{Y_{0:k}} \!\! \hspace{-3mm} b(X_k | X_0, Y_{0:k}) pr(Y_{0:k})dY.
\end{align*}
This guarantees that any execution of the returned plan satisfies the chance constraints. 
%
With known linear dynamics and measurement models with Gaussian noise, and the feedback control law of Sec. \ref{sec:Control}, the expected belief is Gaussian distributed: $\expBelief(X_k)=\N(\hat{X}_k,\Gamma_k)$. 
As shown in~\cite{Bry2011_BeliefProp},
the covariance $\Gamma_k = \Sigma_k+\Lambda_k$ can be calculated and evolved as sum of the online state uncertainty $\Sigma_k$ inflated by the uncertainty due to a priori unknown measurements $\Lambda_k$. 

The Belief-$\A$ planner \cite{Ho2022_GBT} provides a framework for uncertain single agent motion planning using the expected belief formulation with propagation according to \cite{Bry2011_BeliefProp}. In the following sections, we propose a sampling-based algorithm that adapts Belief-$\A$ for CL-MRMP. We detail specific adaptations for Belief-$RRT$ and Belief-$EST$ in Sec. \ref{sec:Belief A for CL-MRMP}, propose efficient methods for checking the safety constraints in Sec. \ref{sec:Prob Approx}, and provide three different CL biasing methods 
in Sec. \ref{sec:Biasing for CL}.


\subsection{Belief-$\A$ for CL-MRMP}
\label{sec:Belief A for CL-MRMP}
We adapt the Belief-$\A$ framework for $\A=\text{RRT}$ \cite{LaValle1998_RRT} and $\A=\text{EST}$ \cite{Hsu1997_EST} to solve the CL-MRMP problem. Both Belief-$RRT$ and Belief-$RRT$ build a search tree $G$ consisting of nodes $\mathbb{V}$ and edges $\mathbb{E}$. Nodes are the beliefs $\expBelief(X_k)$ (denoted $\expBelief_k$), and edges ($\overline{\expBelief_{k_1}\rightarrow \expBelief_{k_2}}$) consist of the nominal control, state, and measurement: ($\check{U}$, $\check{X}$, $\check{C}$). The algorithm follows the same steps as the original RRT and EST algorithms (selection, extension, and validation), as presented in Algorithm \ref{alg: Sampling-Based Planner}.

\vspace{-3mm}
\begin{algorithm}
\KwIn{$\X$, $\{\X_G^i\}$, $\W_O$, $N$}
\KwOut{$G$}
$G\leftarrow (\mathbb{V}\leftarrow \{b_0\}, \mathbb{E}\leftarrow\emptyset)$\;
\For{$N$ iterations}{
    $\expBelief_{select}\leftarrow\textsc{SelectBelief}()$\;
    $(\expBelief_{new}, \overline{\expBelief_{select}\rightarrow \expBelief_{new}}) \leftarrow\textsc{ExtendBelief}()$\;
    \If{$\textsc{ValidBelief}(\expBelief_{new})$}{
        $\mathbb{V}\leftarrow\mathbb{V}\cup \{\expBelief_{new}\}$\;
        $\mathbb{E}\leftarrow\mathbb{E}\cup \{\overline{\expBelief_{select}\rightarrow \expBelief_{new}}\}$\;
    }
}
\Return $G$
\caption{Belief-$\A$ for CL-MRMP}
\label{alg: Sampling-Based Planner}
\end{algorithm}
\vspace{-4mm}

RRT selects a state by uniformly sampling the state space ($\textsc{UniformSampleBelief}$), then selecting the closest node to the sampled state ($\textsc{Nearest}$). EST maintains a sparsity pdf over all the nodes in the tree, the selected node is sampled from this PDF ($\textsc{SparsityPDFsample}$).
Both the generic RRT and EST algorithms can be made more efficient by biasing, e.g., classic goal biasing. We propose various methods to bias toward CL with rate  $\epsilon$ in Sec. \ref{sec:Biasing for CL}: RRT is biased by modifying the sampled state ($\textsc{BiasedSampleBelief}$), 
whereas EST is biased by sampling from a 
pdf ($\textsc{BiasedPDFsample}$). 

\vspace{-3mm}
\begin{algorithm}
\KwIn{$\X$, $\epsilon$}
\KwOut{$\expBelief_{select}$}
$p\leftarrow \textsc{StandardUniformSample()}$\;
\If{$p<\epsilon$}{
    $\expBelief_{sample}\leftarrow \textsc{BiasedSampleBelief}()$\;
} \Else{
$\expBelief_{sample}\leftarrow \textsc{UniformSampleBelief}()$\;
}
$\expBelief_{select}\leftarrow \textsc{Nearest}(G, \expBelief_{sample})$\;
\Return $\expBelief_{select}$
\caption{\textsc{SelectBelief-RRT}()}
\label{alg:SelectBelief-RRT}
\end{algorithm}
\vspace{-4mm}

The belief validation function $\textsc{ValidBelief}$ checks for collisions with each obstacle and each other robots. Note that exactly checking the safety constraint requires integration over the belief, which is generally intractable, suitably efficient approximations are described in Sec~\ref{sec:Prob Approx}.

The $\textsc{ExtendBelief}$ function uses the belief propagation equations of \cite{Bry2011_BeliefProp}, but with the measurement matrix constructed 
to respect the CL chance constraint, i.e. $C_k$ is constructed only from the $C_k^{ij,ext}$ for robots $i$,$j$ that satisfy the constraint on \eqref{eq:CL probability}. As with checking the probabilistic safety constraint, exact evaluation of the CL constraint is intractable, we provide a detailed description of an efficient implementation in Section \ref{sec:Prob Approx} with the $\textsc{ExtEnabled}$ function.

\vspace{-3mm}
\begin{algorithm}
\KwIn{$\X$, $pdf_G$, $\epsilon$}
\KwOut{$\expBelief_{select}$}
$p\leftarrow \textsc{StandardUniformSample}()$\;
\If{$p<\epsilon$}{
    $\expBelief_{select}\leftarrow \textsc{BiasedPDFsample}()$\;
} \Else{
    $\expBelief_{select}\leftarrow \textsc{SparsityPDFsample}()$\;
}
\Return $\expBelief_{select}$
\caption{\textsc{SelectBelief-EST}()}
\label{alg:SelectBelief-EST}
\end{algorithm}
\vspace{-4mm}

Under the theoretical results in \cite{Ho2022_GBT}, Belief-$\mathcal{A}$ inherits the completeness properties of the underlying algorithm (RRT and EST in our case), therefore our algorithm is probabilistically complete. Correctness is preserved by correct evaluation of the chance constraints in the \textsc{ValidBelief} and \textsc{ExtEnabled} functions, which is discussed in Section \ref{sec:Prob Approx}. If \textsc{ValidBelief} and \textsc{ExtEnabled} use conservative approximations, then our algorithm is complete only with respect to the approximation.

\section{Efficient Validation of Chance Constraints}
\label{sec:Prob Approx}

Validation ensures correctness by requiring that any node added to the tree satisfies the chance constraints. Because this operation is called in every planning iteration it must be efficient. As discussed earlier, exact evaluation of the chance constraints is intractable. In this section, we provide suitably efficient and conservative approximations.

For all of the methods described in the subsequent sections we rely on the expedient of probability contours rather than exact integration of probability mass. For a random variable $z\in\mathbb{R}^n$ that is Gaussian distributed as $z\sim \N(\mu,\Sigma)$, the contour containing $p$ probability mass is an ellipse defined by the spectral decomposition of the covariance $\Sigma$ and a scaling factor $\alpha$ calculated from the inverse $\chi^2$ distribution.

\subsection{Probability allocation}
We allocate the safety probability from \eqref{eq:cc obs} among robot-obstacle collision, robot-robot collision, and violation of the exteroceptive measurement condition as 
$1-p_\text{safe} = p_{obs} + p_{rob} + p_{\neg CL}$, and correspondingly separate the safety constraint:
\vspace{-3mm}
\begin{equation*}
    P^{i_k}_O \leq p_{obs}, \quad 
    \sum_{j=1, j \neq i}^{N_A} P^{ij_k}_{coll} \leq  p_{rob}, \quad 
    \sum_{j=1, j \neq i}^{N_A} P_{\neg CL}^{ij_k} \leq p_{\neg CL}.
    \vspace{-1mm}
\end{equation*}
This formulation vastly simplifies checking for constraint violation by fixing a threshold for each type of violation.

\subsection{Robot-Obstacle Collision Checking}
Our validity checking 
uses the simplest 
and most computationally efficient
collision checking 
method
from \cite{Theurkauf2024_CCKCBS}. The robot body $\B^i$ is bounded by a sphere $\B^i\subset\Scal^i_{rob}$; a safety contour is then defined containing $p_\text{safe}$ probability mass for each robot. That elliptical contour is then bounded by a sphere which is inflated by the radius of $\Scal^i_{rob}$. If the inflated sphere does not intersect with any obstacles, the obstacle chance constraint is satisfied. We directly apply this method with the marginal $\expBelief(x_k^i)$. 
The simplicity of this method is well-suited to the complex CL-MRMP problem, especially with the more 
expensive biasing methods in Sec. \ref{sec:Biasing for CL}.

\subsection{Robot-Robot Collision Checking}

For the next two sections, we consider the joint distribution $\expBelief(X_k)$ to leverage information from correlation, and use the difference 
$\xeuc_k^{ij} = \proj{\W}(x_k^i)-\proj{\W}(x_k^j)$, which is distributed as $\expBelief(\xeuc_k^{ij})=\N(\hat{\xeuc}_k^{ij}, \sigma_k^{ij})$, with mean $\hat{\xeuc}_k^{ij} = \hat{\xeuc}_k^i-\hat{\xeuc}_k^j$ and covariance $\sigma_k^{ij} = \Sigma_k^i + \Sigma_k^j - 2\Sigma_k^{ij}$. The term $\Sigma_k^{ij}$ captures the off-diagonal terms of $\Sigma_k$ correlating robot $i$'s and $j$'s Euclidean state estimates. Note $\xeuc_k^{ij}$ is not the \emph{distance} $r^{ij}$ from Sec. \ref{sec:prob obj}, which is nonlinear in $x_k^i$ and therefore non-Gaussian.

Define the bounding spheres for each robot $\B^i\subset\Scal^i_{rob}$ and $\B^j\subset\Scal^j_{rob}$, with radii $r_i$ and $r_j$ respectively. Collision is then determined as $(\xeuc_k^i, \xeuc_k^j)\in \X_{coll}^{ij} \iff \|\xeuc_k^i-\xeuc_k^j\|\leq r_i+r_j$, and thus the set of collision states is a sphere $\R_r$ of radius $r_i+r_j$ centered on the origin in $\xeuc_k^{ij}$ space. The collision probability is the intractable integral $ P((\xeuc_k^i, \xeuc_k^j)\in \X_{coll}^{ij}) = \int_{\R_r} \expBelief(\xeuc_k^{ij})(s)ds$. 
Next define a probability ellipsoid on $\expBelief(\xeuc_k^{ij})$ containing $1-p_{rob}$ probability mass, and bound it with sphere $\Scal_{p_{rob}}$ such that $P(\xeuc^{ij}_k\in\Scal_{p_{rob}})\geq 1-p_{rob}$. If this sphere does not intersect with the ball $\R_r$, we can conclude that it is entirely subsumed by the excluded set $\R_r\subset\tilde{\Scal}_{p_{rob}}$,  $\tilde{\Scal}_{p_{rob}} = \{\xeuc_k^{ij} \notin \Scal_{p_{rob}}\}$, and therefore  $P(\xeuc_k^{ij} \in \R_r) \leq p_{rob}$. This validation procedure is presented in Alg.~\ref{alg:RobotRobotCollision}.

\vspace{-3mm}
\begin{algorithm}
\KwIn{$\expBelief^{ij} = \N(\hat{\xeuc}^{ij},\sigma^{ij})$, $r_i$, $r_j$}
$\lambda_{max}\leftarrow \textsc{MaxEigenvalue}(\sigma^{ij})$\;
$\alpha \leftarrow \textsc{inv}\chi^2(p_{rob},2)$\;
$r_{rob}\leftarrow \sqrt{\alpha\lambda_{max}}$\;
\If{$\|\hat{\xeuc}^{ij}\|-r_{rob} > r_i + r_j$}
{
    \Return True\;
}
\Return False
\caption{$\textsc{RobotRobotCollision}$}
\label{alg:RobotRobotCollision}
\end{algorithm}
\vspace{-3mm}

\begin{theorem}
    \label{thm:Robot-robot collision}
    The validity checking Alg.~\ref{alg:RobotRobotCollision} guarantees the satisfaction of the robot-robot collision constraint $P^{ij_k}_{coll} \leq  p_{rob}$ if it returns True.
\end{theorem}
\begin{proof}
Under the definition of the contour $\Scal_{p_{rob}}$ it holds that $P(\xeuc^{ij}_k\in\Scal_{p_{rob}})\geq 1-p_{rob}$, and conversely $P(\xeuc^{ij}_k\notin\Scal_{p_{rob}})< p_{rob}$. It follows that if the sphere $\R_r$ containing all possible robot-robot collision states does not intersect with $\Scal_{p_{rob}}$, then it is subsumed by the excluded set, and therfore $P^{ij_k}_{coll}=P(\xeuc^{ij}_k\in\R_r)<P(\xeuc^{ij}_k\notin\Scal_{p_{rob}})< p_{rob}$.
\end{proof}

\subsection{CL Condition}
Using the same variable $\xeuc^{ij}_k$ as the prior section, we define a sphere centered on the origin containing all states with exteroceptive measurements enabled such that $\R_{ext} = \{\xeuc_k^{ij} \mid \|\xeuc_k^i-\xeuc_k^j\|\leq r_{ext}\}$.
As with the prior section, we find the probability ellipsoid on $\expBelief(\xeuc_k^{ij})$ containing $1-p_{\neg CL}$ probability mass and bound it with sphere $\Scal_{p_{\neg CL}}$ such that $P(\xeuc^{ij}_k\in\Scal_{p_{\neg CL}})\geq 1-p_{\neg CL}$. If the sphere $\R_{ext}$ subsumes $\Scal_{p_{\neg CL}}$, then we can conclude that $P(\| \xeuc_i - \xeuc_j \| > r_{ext})<p_{\neg CL}$. This validation procedure is presented in Alg.~\ref{alg:ExtEnabled}.

\vspace{-3mm}
\begin{algorithm}
\KwIn{$\expBelief^{ij} = \N(\hat{\xeuc}^{ij},\sigma^{ij})$, $r_{ext}$}
$\lambda_{max}\leftarrow \textsc{MaxEigenvalue}(\sigma^{ij})$\;
$\alpha \leftarrow \textsc{inv}\chi^2(p_{rob},2)$\;
$r_{prob}\leftarrow \sqrt{\alpha\lambda_{max}}$\;
\If{$\|\hat{\xeuc}^{ij}\|+r_{prob}<r_{ext}$}
{
    \Return True\;
}
\Return False
\caption{$\textsc{ExtEnabled}$}
\label{alg:ExtEnabled}
\end{algorithm}
\vspace{-4mm}

\begin{theorem}
\label{thm:CL condition}
The validity checking Alg.~\ref{alg:ExtEnabled} guarantees satisfaction of the CL constraint $P_{\neg CL}^{ij_k} \leq p_{\neg CL}$ if it returns True.
\end{theorem}
\begin{proof}
Under the definition of the contour $\Scal_{p_{\neg CL}}$ it follows $P(\xeuc^{ij}_k\in\Scal_{p_{\neg CL}})\geq 1-p_{\neg CL}$, and conversely $P(\xeuc^{ij}_k\notin\Scal_{p_{\neg CL}})< p_{\neg CL}$. It follows that if the sphere $\R_{ext}$ containing all possible CL states subsumes $\Scal_{p_{\neg CL}}$, then the probability of not being in $\R_{ext}$ (no CL) is: $P_{\neg CL}^{ij_k}=P(\xeuc^{ij}_k\notin\R_{ext})<P(\xeuc^{ij}_k\notin\Scal_{p_{\neg CL}})< p_{\neg CL}$.
\end{proof}

\section{Biasing For CL}
\label{sec:Biasing for CL}
To encourage cooperative behaviors in our system, 
we preferentially select nodes 
with robot states where CL is enabled, i.e. nodes where the robots are close together. Three proposed biasing methods are detailed in this section. 
Note that 
biasing necessarily encourages unsafe behavior by increasing the likelihood of robot-robot collisions, thus raising
an interesting tension: 
more effective CL biasing will limit tree growth due to safety violation. 
This makes 
biasing effects difficult to predict, with inconsistencies dependent on the system and environment (as seen in Sec. \ref{sec:Evaluations}).

\begin{figure*}[ht!]
    \centering
    \begin{subfigure}[b]{0.3\textwidth}
        \centering
        \fontsize{5}{5}\selectfont 
        \includegraphics[width=\textwidth,trim={1.4cm 1.2cm 1cm 1.2cm}, clip]{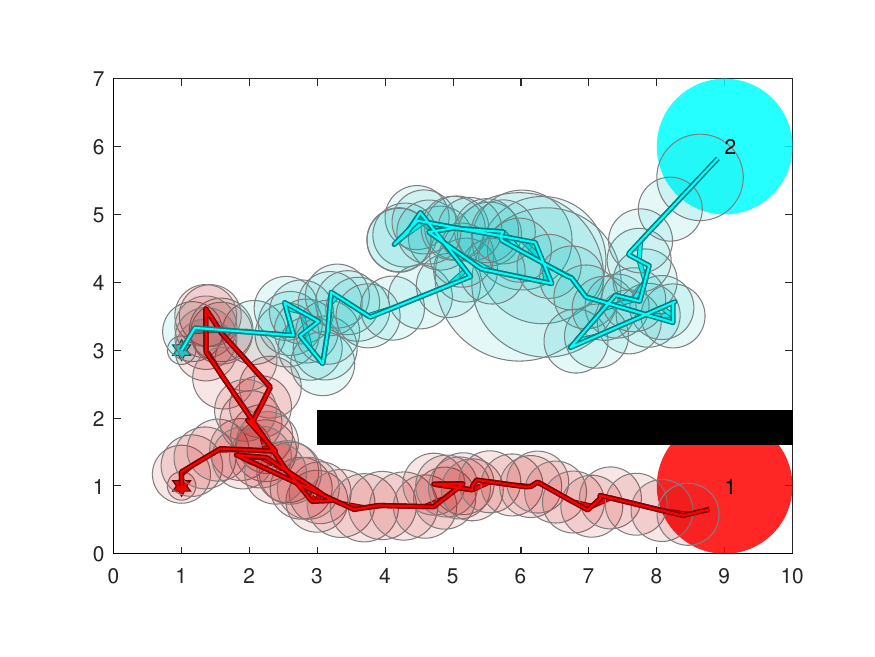}
        \caption{\small Corridor}
        \label{fig:Corridor Environemnt}
    \end{subfigure} \hfill
    \begin{subfigure}[b]{0.17\textwidth}
        \centering
        \fontsize{5}{5}\selectfont 
        \includegraphics[width=\textwidth,trim={3.7cm 0.7cm 3.5cm 0.7cm}, clip]{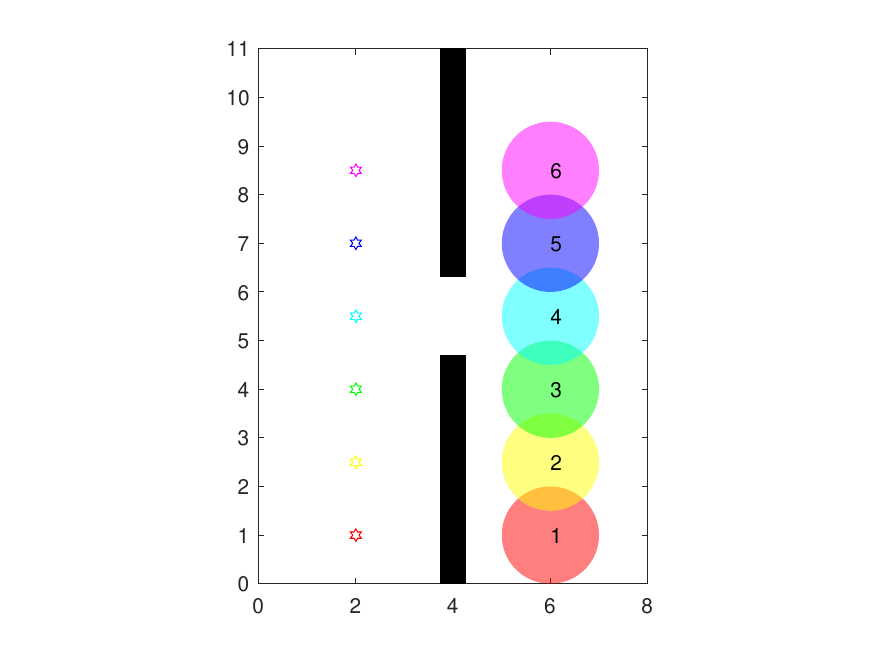}
        \caption{\small Pincer}
        \label{fig:Pincer Environemnt}
    \end{subfigure}\hfill
    \begin{subfigure}[b]{0.22\textwidth}
        \centering
        \fontsize{5}{5}\selectfont 
        \includegraphics[width=\textwidth,trim={2.5cm 0.7cm 2.2cm 0.7cm}, clip]{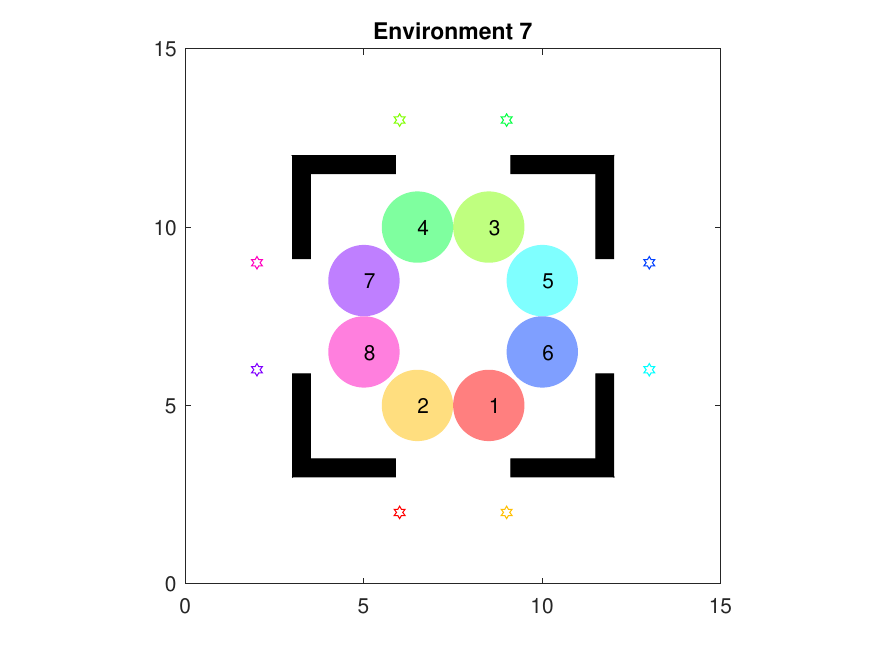}
        \caption{\small Hive}
        \label{fig:Hive Environemnt}
    \end{subfigure}\hfill
    \begin{subfigure}[b]{0.22\textwidth}
        \centering
        \fontsize{5}{5}\selectfont 
        \includegraphics[width=\textwidth,trim={2.5cm 0.7cm 2.2cm 0.7cm}, clip]{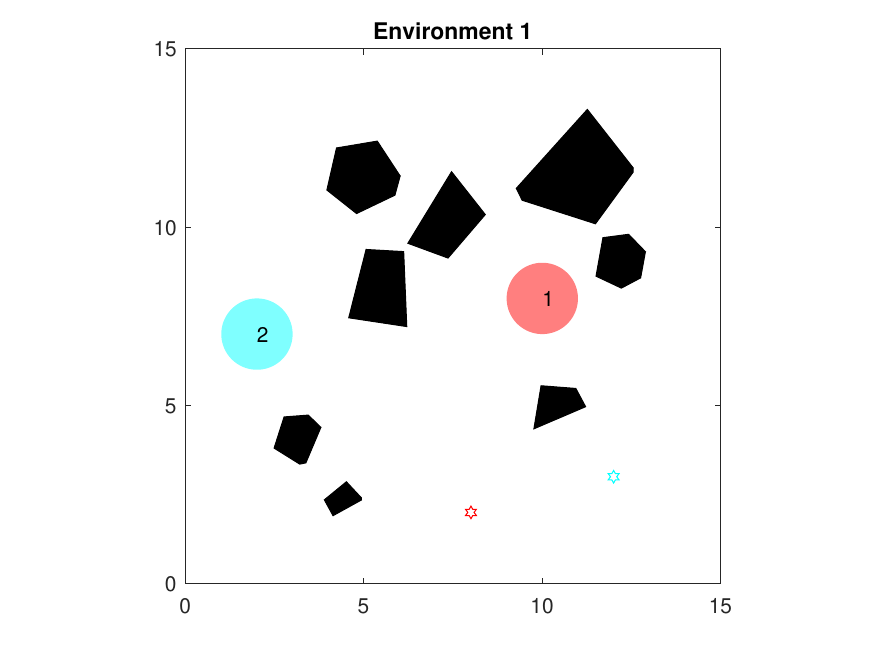}
        \caption{\small Random}
        \label{fig:Random Environemnt}
    \end{subfigure}
    \caption{\small Test Environments}
    \label{fig:envs}
    \vspace{-5mm}
\end{figure*}

\subsection{State Cloning}
\label{sec:cloning}
The first method is implemented in RRT node selection (Alg. \ref{alg:SelectBelief-RRT}) with the $\textsc{BiasedSampleBelief}$ function. The native selection scheme of RRT (sampling then finding the nearest neighbor) is not well suited to bias toward nodes that contain close robots because it relies on the distance \emph{between} nodes, not the distance \emph{within} nodes. Our biased sampler forms sampled beliefs where all robots occupy the same state.
We 
call this `Cloning'. 
Begin by sampling the composed mean, $\hat{X}_{samp}$, and covariance, $\Sigma_{samp}$, as usual. 
Then choose a robot to clone $i_{clone}\in\{1,...,N_A\}$, with projected state $\hat{\xeuc}_k^{i_{clone}} = \proj{\W}(\hat{x}^{i_{clone}})$. Our implementation alternates among all robots, but other methods, e.g. at random, are possible. Then, iterate over all robots replacing their projected state means with the cloned robot's. This gives a composed mean vector with all robots 
at the same projected state. 


This 
is the most computationally efficient and scalable of the proposed biasing methods, being independent of tree size. This allows the planner 
to achieve roughly the same number of iterations within a fixed 
time regardless of 
$\epsilon$. In contrast, the other two methods have a distinct trade-off as increased $\epsilon$ limits tree growth due to computational complexity.

\subsection{Distance Weighting}
\label{sec:weight}
The second biasing method is implemented in EST node selection (Alg.~\ref{alg:SelectBelief-EST}) with the  $\textsc{BiasedPDFsample}$ function. Unlike indirect sampling and nearest neighbor selection in RRT, the likelihood of sampling a given node in EST is determined only by the node's weight. We can therefore directly sample according to the distance between robots within each node by maintaining a biasing pdf. 
For a node $\expBelief^n = \mathcal{N}(\hat{X}_n$, $\Sigma_n)$ with projected Euclidean state $\hat{\xeuc}^i=\proj{\W}(\hat{x}_n^i)$ for each robot, we define node weight $\mathcal{W}(\expBelief^n)$ as: 
\vspace{-2mm}
\begin{equation}
    \label{eq:distance weight}
    \mathcal{W}(\expBelief^n) = \frac{1}{\mathcal{D}(\expBelief^n)}, \quad \mathcal{D}(\expBelief^n) = \sum_{i=1}^{N_A}\sum_{j=1}^{N_A} \| \hat{\xeuc}_i - \hat{\xeuc}_j \|. 
    \vspace{-1mm}
\end{equation}
\noindent A proper PDF over all tree nodes is obtained by normalizing over the node weights. Note the double sum in $\mathcal{D}(\expBelief^n)$ scales poorly with $N_A$, which limits tree growth.

\subsection{Re-Branching}
\label{sec:Rebranching}
Finally, we propose a more complex biasing technique for both RRT and EST that modifies the selected node. This method reshuffles individual robot pairings to form new branches from existing branches; we call this 
`Re-branching'. This method slightly modifies the sampling-based planner framework, as described in Alg. \ref{alg: Sampling-Based Planner Rebranch}.

\vspace{-3mm}
\begin{algorithm}
\KwIn{$\X$, $\{\X_G^i\}$, $\W_O$, $N$, $\epsilon$}
\KwOut{$G$}
$G\leftarrow (\mathbb{V}\leftarrow \{\expBelief_0\}, \mathbb{E}\leftarrow\emptyset)$\;
\For{$N$ iterations}{
    $\expBelief_{select}\leftarrow\textsc{SelectBelief}()$\;
    $p\leftarrow \textsc{StandardUniformSample()}$\;
    \If{$p<\epsilon$}{
        $\expBelief_{select}\leftarrow \textsc{ReBranch}(b_{select})$\;
    }
    $(\expBelief_{new}, \overline{\expBelief_{select}\rightarrow \expBelief_{new}}) \leftarrow\textsc{ExtendBelief}()$\;
    \If{$\textsc{ValidBelief}(\expBelief_{new})$}{
        $\mathbb{V}\leftarrow\mathbb{V}\cup \{\expBelief_{new}\}$\;
        $\mathbb{E}\leftarrow\mathbb{E}\cup \{\overline{\expBelief_{select}\rightarrow \expBelief_{new}}\}$\;
    }
}
\Return $G$
\caption{Sampling-Based Planner, Re-branch}
\label{alg: Sampling-Based Planner Rebranch}
\end{algorithm}
\vspace{-3mm}

We start with the selected node $\expBelief^{sel}_{k^\dagger}$, defining a trajectory in belief space terminating at 
$\expBelief^{sel}_{k^\dagger}$ at time $k^\dagger$. A single target robot $i_{target}\in\{1,...,N_A\}$ is chosen, with projection of the mean from $\expBelief^{sel}_{k^\dagger}$ into Euclidean space $\hat{\xeuc}^{i_{target}}_{k^\dagger}$ (similar to Cloning). We then search the entire motion tree for the robot whose projected Euclidean state is closest to $\hat{\xeuc}^{i_{target}}_{k^\dagger}$ at time $k^\dagger$; this is the new paired robot $i_{pair}$ which corresponds to belief node $\expBelief^{pair}_{k^\dagger}$. We then form a new branch by replacing the nominal control edges for robot $i_{pair}$ in the original coupled trajectory for $\expBelief^{sel}_{k^\dagger}$ and re-propagating the trajectory. This results in the new coupled belief node $\expBelief^{re}_{k^\dagger}$ where robots $i_{target}$ and $i_{pair}$ have the closest possible Euclidean distances at time $k^\dagger$ of the existing tree states.

Re-branching is conceptually straightforward and offers promising results, particularly for smaller problems. However, its implementation presents challenges that impact scalability due to three key factors. First, it requires re-propagating and validating new branch edges. Second, it demands precise time synchronization of nodes to form new states, introducing additional states whenever the intermediate nodes of the selected and paired nodes differ in time. Third, we must search the tree over individual robot states rather than the \emph{coupled} states of the motion tree. We address this by maintaining nearest neighbor data structures over the projected states for each robot. While these factors introduce computational overhead for large search trees, our results demonstrate the potential of Re-branching to efficiently handle smaller-scale problems.

\section{Evaluations}
\label{sec:Evaluations}

We evaluate our algorithm across four environments shown in Fig.~\ref{fig:envs}
to assess performance across different planners and biasing techniques. Robot start locations (stars) and goal regions (shaded circles) are predefined.
We implemented the planners (Belief-RRT and -EST) in OMPL \cite{sucan2012the-open-motion-planning-library} and ran all experiments on an Intel Core i7-12700K CPU with 32GB~RAM. 


In all cases, each robot has 2D dynamics given by: $A_k^i = B_k^i = I_{2\times 2}$, $Q_k^i = 0.01I_{2\times 2}$. A subset of robots (corresponding to odd indices $i$) has access to proprioceptive measurements, making the system unobservable in the absence of exteroceptive measurements. Primary results are reported with two robots, scaling results are reported up to six robots, with additional results in the Appendix.
The 
measurement model is given by 
$C_k^{i,prop}=I_{2\times 2}$ and 
$C_k^{ij,ext}=[I_{2\times 2}, -I_{2\times 2}]$. The chance constraints are set to $p_{obs}=p_{rob}=p_{\neg CL}=0.05$.


\subsection{Illustrative Examples}
Figs.~\ref{fig:Random trajectory} and~\ref{fig:Corridor Environemnt} show two example trajectories for the Random and Corridor environments, respectively.
The cyan robot states are unobservable without CL, as shown in Fig.~\ref{fig:Random trajectory zoom}. 
In Fig.~\ref{fig:Random trajectory} the red robot must divert to enable CL so that the cyan robot 
reaches its goal. In Fig.~\ref{fig:Corridor Environemnt}, the cyan robot diverts to remain close to the red robot throughout its trajectory and keep its uncertainty small.
An online method that only drives toward the goal would fail to find these cooperative behaviors, and a decoupled approach that does not account for CL would fail to plan for the 
An online method that only drives toward the goal would fail to find these cooperative behaviors, and a decoupled approach that does not account for CL would fail to plan for the 
cyan robot. Our algorithm finds both plans within 2 minutes.





\begin{figure*}
    \centering
    \begin{subfigure}[b]{0.9\textwidth}
        \centering
        \includegraphics[width=\textwidth]{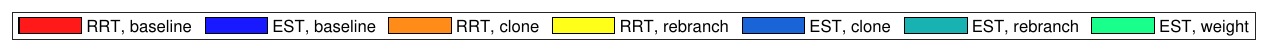}
    \end{subfigure}
    \newline
    \begin{subfigure}[b]{0.32\textwidth}
        \centering
        \includegraphics[width=\textwidth,trim={2.8cm 3cm 3.5cm 3.2cm}, clip]{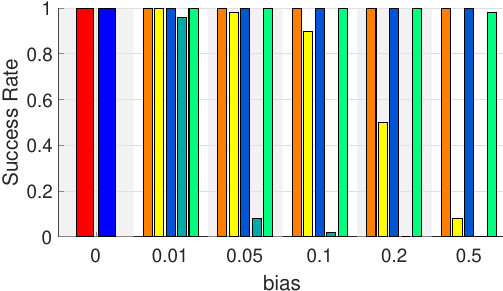}
        \caption{\small Corridor, Success Rate}
        \label{fig:Corridor success rate}
    \end{subfigure}
    \begin{subfigure}[b]{0.32\textwidth}
        \centering
        \includegraphics[width=\textwidth,trim={2.7cm 3cm 3.5cm 3.2cm}, clip]{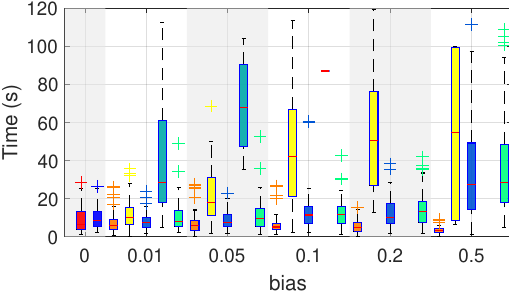}
        \caption{\small Corridor, Solution Time}
        \label{fig:Corridor solution time}
    \end{subfigure}
    \begin{subfigure}[b]{0.32\textwidth}
        \centering
        \includegraphics[width=\textwidth,trim={2.8cm 3cm 3.5cm 3cm}, clip]{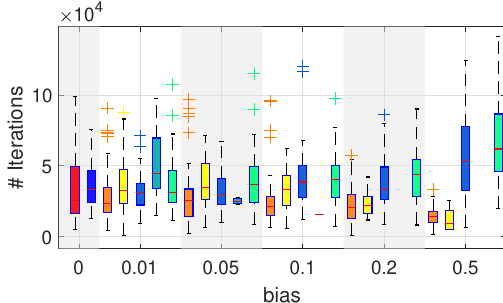}
        \caption{\small Corridor, Number Iterations}
        \label{fig:Corridor number iterations}
    \end{subfigure}
    \newline
    \begin{subfigure}[b]{0.32\textwidth}
        \centering
        \includegraphics[width=\textwidth,trim={2.8cm 3cm 3.5cm 3.2cm}, clip]{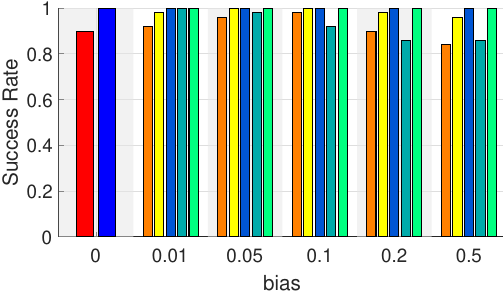}
        \caption{\small Hive, Success Rate}
        \label{fig:Hive2 success rate}
    \end{subfigure}
    \begin{subfigure}[b]{0.32\textwidth}
        \centering
        \includegraphics[width=\textwidth,trim={2.7cm 3cm 3.5cm 3.2cm}, clip]{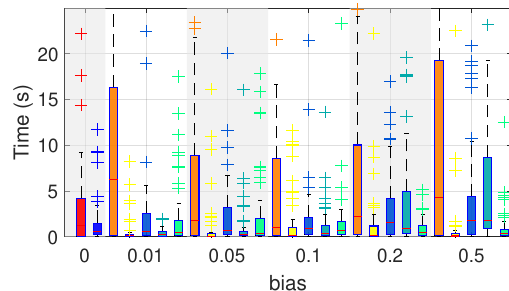}
        \caption{\small Hive, Solution Time}
        \label{fig:Hive2 solution time}
    \end{subfigure}
    \begin{subfigure}[b]{0.32\textwidth}
        \centering
        \includegraphics[width=\textwidth,trim={2.8cm 3cm 3.5cm 3cm}, clip]{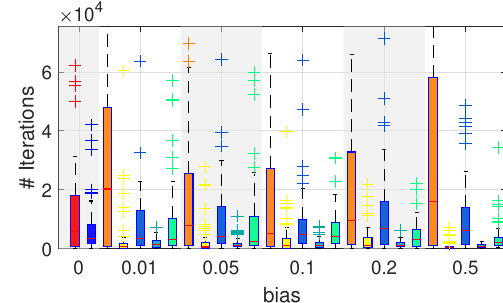}
        \caption{\small Hive, Number Iterations}
        \label{fig:Hive2 number iterations}
    \end{subfigure}
    \newline
    \begin{subfigure}[b]{0.32\textwidth}
        \centering
        \includegraphics[width=\textwidth,trim={2.8cm 3cm 3.5cm 3.2cm}, clip]{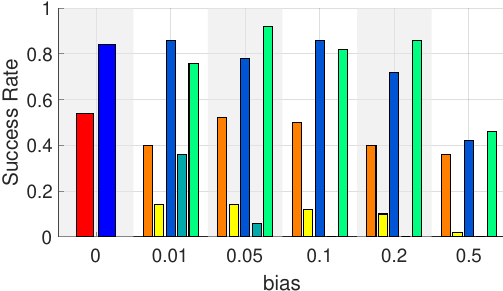}
        \caption{\small Random, Success Rate}
        \label{fig:Random success rate}
    \end{subfigure}
    \begin{subfigure}[b]{0.32\textwidth}
        \centering
        \includegraphics[width=\textwidth,trim={2.7cm 3cm 3.5cm 3.2cm}, clip]{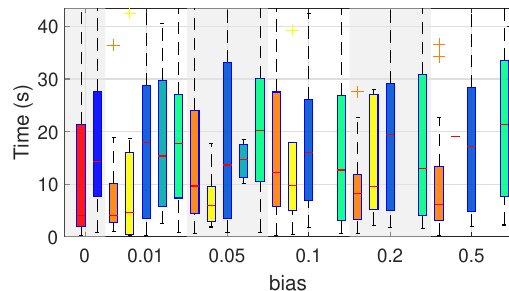}
        \caption{\small Random, Solution Time}
        \label{fig:Random solution time}
    \end{subfigure}
    \begin{subfigure}[b]{0.32\textwidth}
        \centering
        \includegraphics[width=\textwidth,trim={2.8cm 3cm 3.5cm 3cm}, clip]{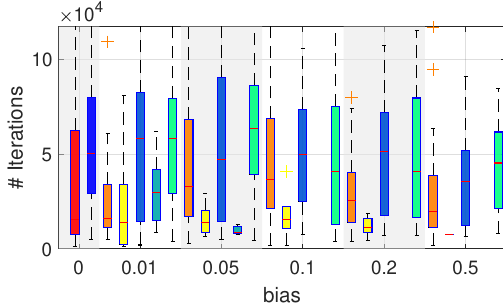}
        \caption{\small Random, Number Iterations}
        \label{fig:Random number iterations}
    \end{subfigure}
    \caption{\small Two Robot Environments}
    \label{fig:2 Robot Environments}
\end{figure*}

\subsection{Benchmarks}
We use benchmarks to compare Belief-$RRT$ and Belief-$EST$ planners with biasing rates ($\epsilon$) from $0.01$ to $0.5$, alongside a no-biasing baseline. Each instance is run $50$ times, with a planning time of 1 minute for `Hive' and 2 minutes for others. We report success rates, computation time, and iteration counts, summarized in Figs.~\ref{fig:2 Robot Environments} and~\ref{fig:4}, where all plots follow the same legend.

Overall, the results show that our proposed algorithm reliably finds solutions in each environment with both RRT and EST variants, and biasing techniques better suited to some over others. In particular, we observe that Re-branching can improve planning time and iterations in simple environments, and RRT with Cloning scales well with $N_A$. Detailed discussion is provided below. 


\textbf{A Note on the Extended Results: }
Our main results are summarized in Figs.~\ref{fig:2 Robot Environments} and \ref{fig:4}, however we include all our additional results in the Appendix. These include additional plots of the robot-robot collision rate, robot-obstacle collision rate, and tree size for the Random environment in Fig.~\ref{fig:Random_Results}, as well as the Hive environment for 2, 3, and 4 robots in Figs.~\ref{fig:Hive 2 Results}, \ref{fig:Hive 3 Results}, and \ref{fig:Hive 4 Results}. We include abbreviated results for 5 and 6 robots in the Hive environement, Figs.~\ref{fig:Hive5 Results} and \ref{fig:Hive6 Results}. We additionally tested the Hive environment with 7 and 8 robots, but the success rate was too low to draw any useful conclusions. We also include plots of the success rate, solution time, and number of iterations for the Pincer environment for 2, 3, 4, and 5 robots in Figs.~\ref{fig:Pincer2 Results}, ~\ref{fig:Pincer3 Results}, ~\ref{fig:Pincer4 Results}, and ~\ref{fig:Pincer5 Results}. We additionally tested the Pincer environment with 6 and 7 robots, but the success rate was too low to draw any useful conclusions.

\textbf{CL vs. Robot-robot collision: }
To illustrate the tension introduced between CL and robot-robot collisions, 
consider the boxplots of the proportion of nodes rejected for robot-obstacle collision (Fig.~\ref{fig:Random obs coll}) and robot-robot collision (Fig.~\ref{fig:Random robot coll}) for the Random environment. 
As $\epsilon$ increases, 
robots are drawn closer together, causing more robot-robot collisions. This is particularly true of 
Re-branching, 
indicating that (despite the high computation cost) it is the most effective CL biasing method. 
We can see further evidence of this in the simpler 2, 3, and 4-robot Hive environments, Figs.~\ref{fig:Hive 2 Results}, \ref{fig:Hive 3 Results}, \ref{fig:Hive 4 Results}, however the collision rates for the Hive are more robust to CL biasing compared to the Random environment, resulting in higher success rates for Re-branching.

\begin{figure*}
    \centering
    \begin{subfigure}[b]{0.32\textwidth}
        \centering
        \includegraphics[width=\textwidth,trim={2.8cm 3cm 3.5cm 3cm}, clip]{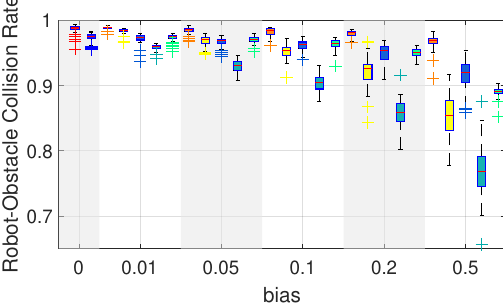}
        \caption{\small Random, Robot-Obstacle Collision Rate}
        \label{fig:Random obs coll}
    \end{subfigure}
    \begin{subfigure}[b]{0.32\textwidth}
        \centering
        \includegraphics[width=\textwidth,trim={2.8cm 3cm 3.5cm 3cm}, clip]{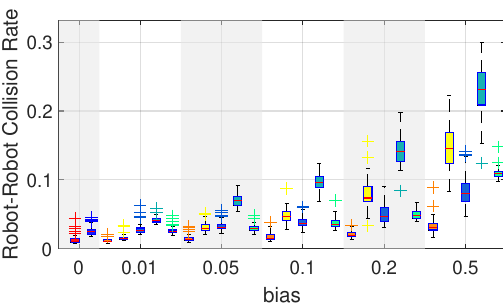}
        \caption{\small Random, Robot-Robot Collision Rate}
        \label{fig:Random robot coll}
    \end{subfigure}
    \begin{subfigure}[b]{0.32\textwidth}
        \includegraphics[width=\textwidth,trim={2.8cm 3cm 3.5cm 3cm}, clip]{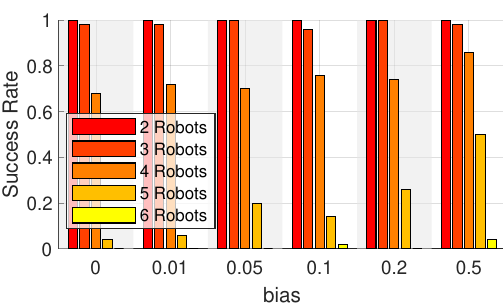}
        \caption{\small Pincer, RRT Cloning Method}
        \label{fig:Pincer RRT clone}
    \end{subfigure}
    \caption{\small Benchmarking results for (a)-(b) collision rates of 2 robots in Random Env., and (c) success rates for 2-6 robots in Pincer Env.}
    \label{fig:4}
    \vspace{-2mm}
\end{figure*}

\textbf{Re-branching:} 
Re-branching is an effective biasing technique in that it can efficiently find a solution for simple cases.
In particular, Re-branching in the Hive environment shows a substantial decrease in solution times (Fig.~ \ref{fig:Hive2 solution time}) and number of iterations (Fig.~\ref{fig:Hive2 number iterations}) from the baseline compared to all other techniques for the same success rate (Fig.~\ref{fig:Hive2 success rate}). However, because Re-branching is computationally intensive its performance can drop rapidly as more robots are added or for complex environments. We can see this in the 2-robot Pincer environment Fig.~\ref{fig:Pincer2 Results}, where the computation time for EST with Re-branching drastically increases and the success rate plummets. We can also see that Re-branching results in larger trees for RRT in the Random and 3-robot Hive environments, see Figs.~\ref{fig:Random_Results}, \ref{fig:Hive 3 Results}, which likely compounds the computation inefficiency and results in especially low success rates in those environments. This indicates that a more efficient re-branching method could be a promising future direction. 

\textbf{Cloning:} 
RRT with Cloning performs well in the Pincer environment for larger $N_A$, increasing the success rate from $0.86$ to $0.68$ for four robots, and $0.04$ to $0.5$ for five robots (Fig.~\ref{fig:Pincer RRT clone}). 
This indicates that the simplicity of the Cloning method can be quite effective for larger robot teams.
Also note that while EST generally does worse than RRT in the Pincer environment, for the three robot case EST with Cloning increases the success from from $0.76$ to $0.98$ while significantly decreasing the solution time and number of robots, as seen in Fig.~\ref{fig:Pincer3 Results} in the Appendix. However, all EST methods failed to find any plans for greater than three robots in the Pincer environment.


\textbf{Distance Weighting:}
The Distance Weighting technique has the most ambiguous effect on performance, with no obvious advantages anywhere in Fig.~\ref{fig:2 Robot Environments}.
The increased computation of Weighting likely outweighs any benefits of the biasing technique, especially with larger $N_A$.

However, in the Random environment it seems to provide modest improvement in success rate, and does at least as well as Cloning, see Figs.~\ref{fig:Random success rate}, \ref{fig:Random solution time}, \ref{fig:Random number iterations}. 
It has a similar effect for three and four robots in the Hive environment, see Figs.~\ref{fig:Hive 3 Results} and \ref{fig:Hive 4 Results} respectively in the Appendix.

This underperformance is unintuitive: Distance Weighting directly reasons over the distances between robots within a node, rather than indirectly biasing like the Cloning technique. It is possible that the increased computation of Weighting outweighs any benefits of the biasing technique, especially with larger numbers of robots
(this can be seen in the high computation times of Weighting vs Cloning especially in Fig.~\ref{fig:Pincer3 Results}).

\textbf{Planner Type: }
Finally, note that EST generally performs better in the Random and Hive environments, while RRT performs better in the Pincer environment. We were able to find plans for 5 and 6 robots in the Hive environemnt with EST, but not RRT, Figs.~\ref{fig:Hive5 Results}, \ref{fig:Hive6 Results}, albeit with low success rates. We were able to find plans for 4 and 5 robots in the Pincer environment with RRT, but not EST, Figs.~\ref{fig:Pincer4 Results}, \ref{fig:Pincer5 Results}. Both perform similarly well in the Corridor environment. This possibly due to the topology of the environments. RRT tends to pull tree growth to unexplored parts of the state space (via Voronoi biasing), while EST more exhaustively explores the space outward from the initial (via sparsity). In the Pincer environemnt RRT more successfully pulls the search tree through the passage and explores the goal side, whereas EST is slow to explore the entire space. On the other hand, the Hive and Random environments do not require extensive searches (especially the Hive), and therefore the exhaustive EST search more quickly finds satisfying plans. The Corridor environment is not particularly suited to either method of tree expansion.

\section{Conclusion}
\label{sec:Conclusion}
We consider the uncertain CL-MRMP problem and propose an algorithm that returns guaranteed safe plans to goal locations. Our proposed biasing techniques improve performance by encouraging exploration of cooperative behaviors, and we study their effectiveness in different scenarios. While our algorithm does not scale well beyond five robots, it reliably finds plans in scenarios that require CL where online methods would fail. Future work will investigate ways to decouple planning to improve scalability.

\bibliographystyle{plain}
\bibliography{references}

\appendix


    

\begin{figure*}[th!]
    \centering
    \begin{subfigure}[b]{0.9\textwidth}
        \centering
        \includegraphics[width=\textwidth]{figures/legend_crop.pdf}
    \end{subfigure}
    \begin{subfigure}[b]{0.32\textwidth}
        \centering
        \includegraphics[width=\textwidth,trim={0cm 0.7cm 0cm 0cm}, clip]{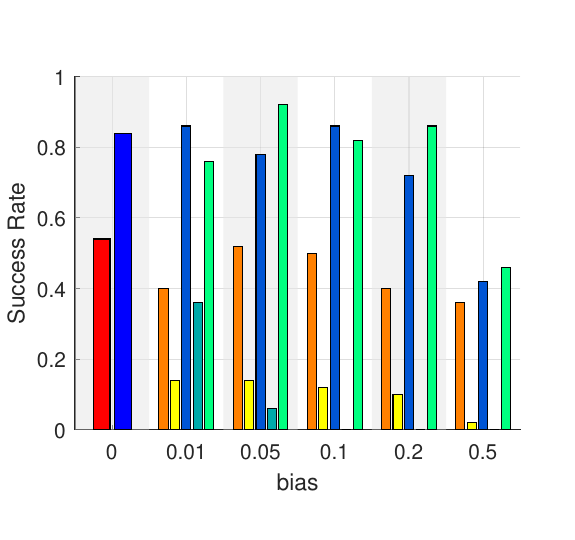}
    \end{subfigure}
    \begin{subfigure}[b]{0.32\textwidth}
        \centering
        \includegraphics[width=\textwidth]{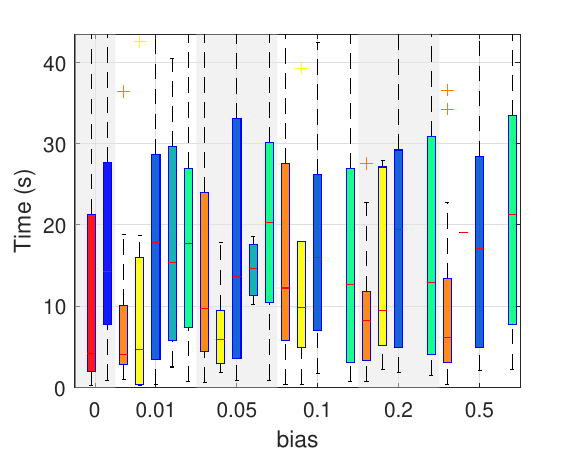}
    \end{subfigure}
    \begin{subfigure}[b]{0.32\textwidth}
        \fontsize{5}{5}\selectfont 
        \includegraphics[width=\textwidth]{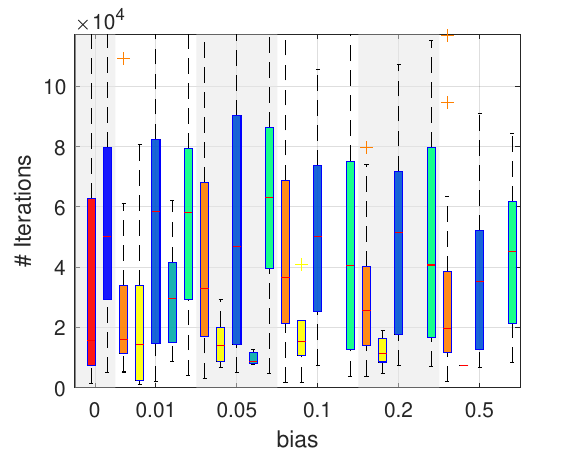}
    \end{subfigure}
    \newline
    \begin{subfigure}[b]{0.32\textwidth}
        \fontsize{5}{5}\selectfont 
        \includegraphics[width=\textwidth]{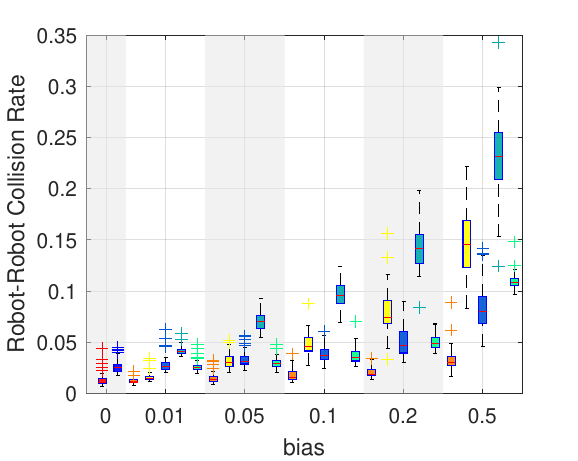}
    \end{subfigure}
    \begin{subfigure}[b]{0.32\textwidth}
        \fontsize{5}{5}\selectfont 
        \includegraphics[width=\textwidth]{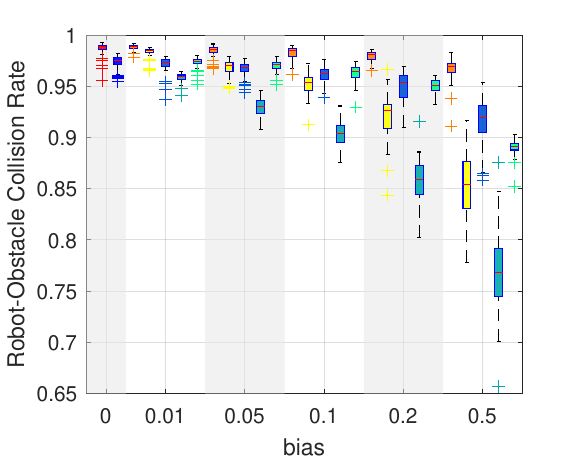}
    \end{subfigure}
    \begin{subfigure}[b]{0.32\textwidth}
        \fontsize{5}{5}\selectfont 
        \includegraphics[width=\textwidth]{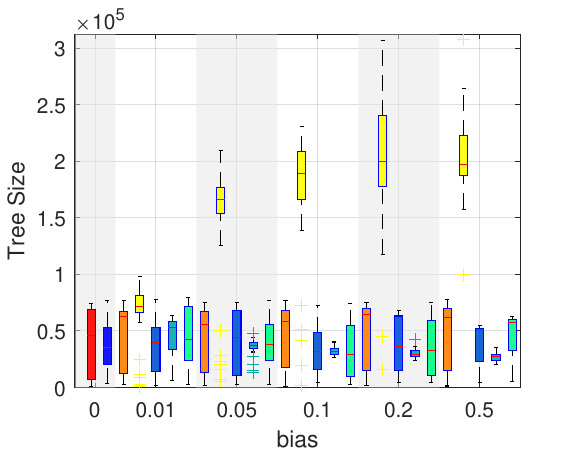}
    \end{subfigure}
    \caption{Random Results}
\label{fig:Random_Results}
\end{figure*}



\begin{figure*}
    \centering
    \begin{subfigure}[b]{0.32\textwidth}
        \centering
        \includegraphics[width=\textwidth,trim={0cm 0.7cm 0cm 0cm}, clip]{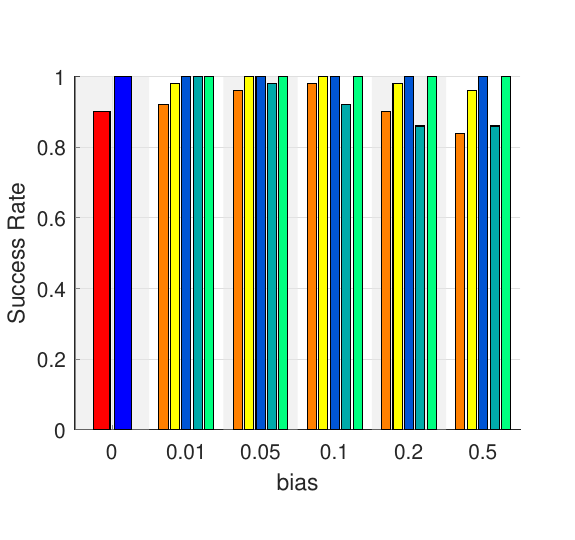}
    \end{subfigure}
    \begin{subfigure}[b]{0.32\textwidth}
        \centering
        \includegraphics[width=\textwidth]{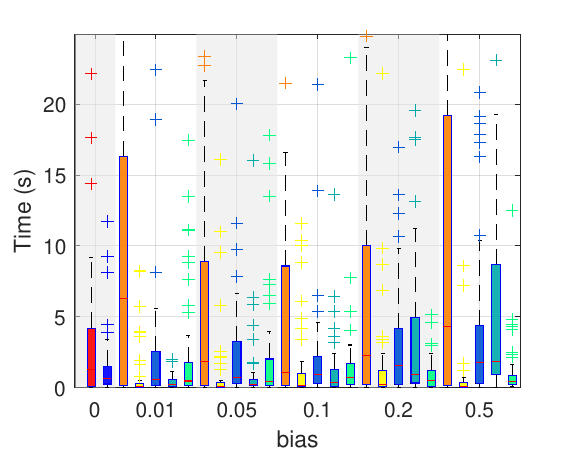}
    \end{subfigure}
    \begin{subfigure}[b]{0.32\textwidth}
        \fontsize{5}{5}\selectfont 
        \includegraphics[width=\textwidth]{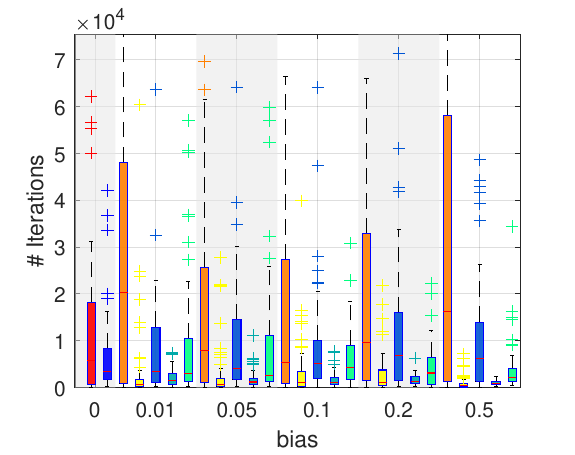}
    \end{subfigure}
    \begin{subfigure}[b]{0.32\textwidth}
        \fontsize{5}{5}\selectfont 
        \includegraphics[width=\textwidth]{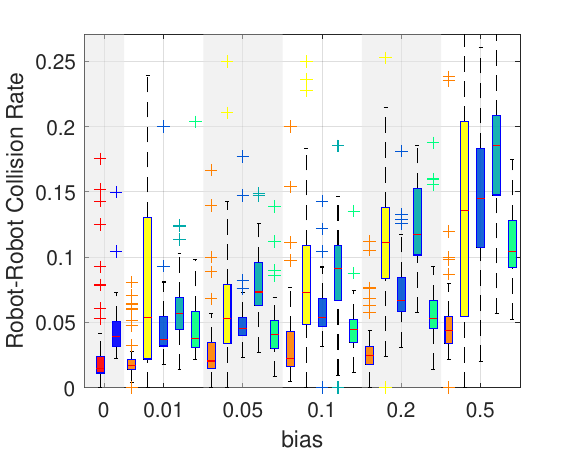}
    \end{subfigure}
    \begin{subfigure}[b]{0.32\textwidth}
        \fontsize{5}{5}\selectfont 
        \includegraphics[width=\textwidth]{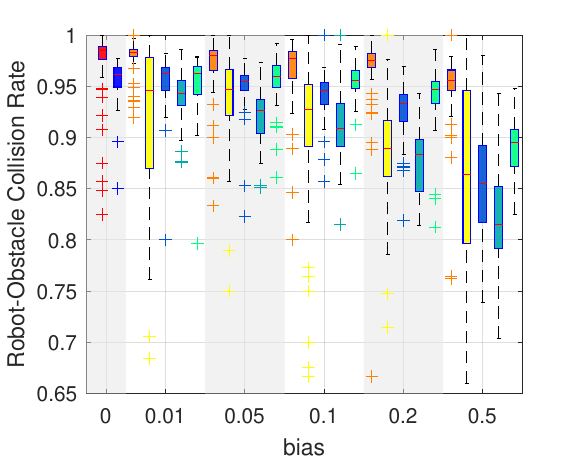}
    \end{subfigure}
    \begin{subfigure}[b]{0.32\textwidth}
        \fontsize{5}{5}\selectfont 
        \includegraphics[width=\textwidth]{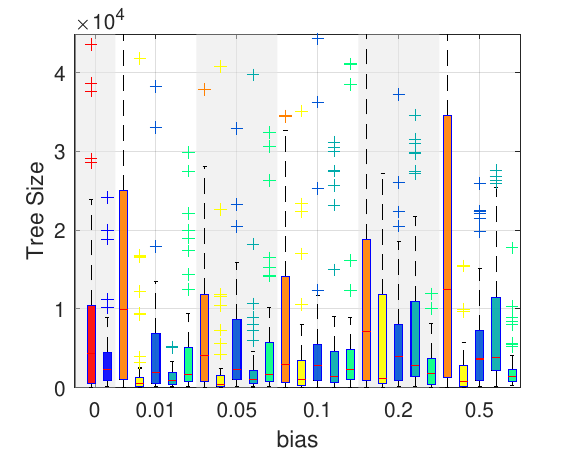}
    \end{subfigure}
    \caption{Hive 2 Results}
    \label{fig:Hive 2 Results}
\end{figure*}

\begin{figure*}
    \centering
    \begin{subfigure}[b]{0.9\textwidth}
        \centering
        \includegraphics[width=\textwidth]{figures/legend_crop.pdf}
    \end{subfigure}
    \begin{subfigure}[b]{0.32\textwidth}
        \centering
        \includegraphics[width=\textwidth,trim={0cm 0.7cm 0cm 0cm}, clip]{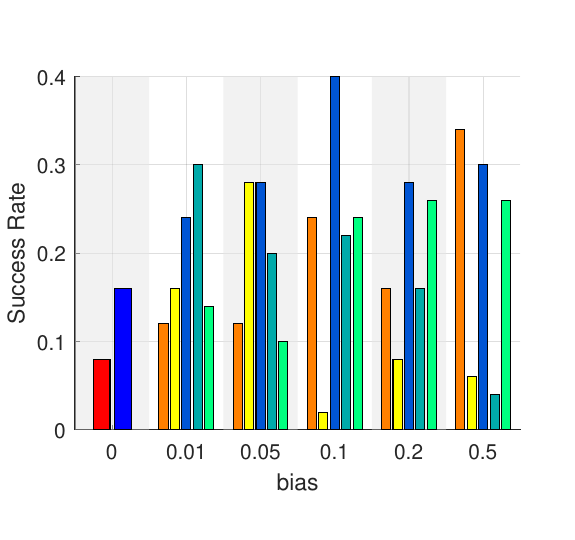}
    \end{subfigure}
    \begin{subfigure}[b]{0.32\textwidth}
        \centering
        \includegraphics[width=\textwidth]{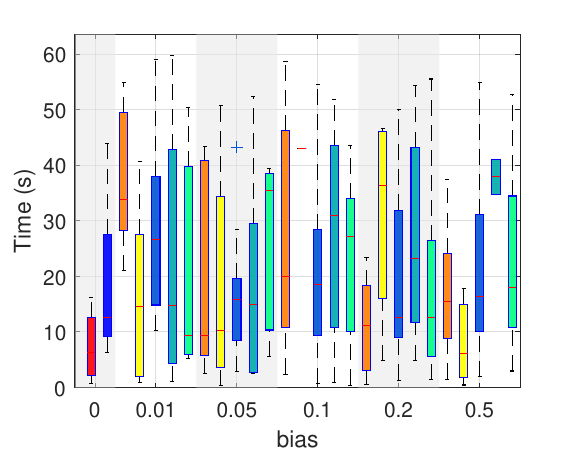}
    \end{subfigure}
    \begin{subfigure}[b]{0.32\textwidth}
        \fontsize{5}{5}\selectfont 
        \includegraphics[width=\textwidth]{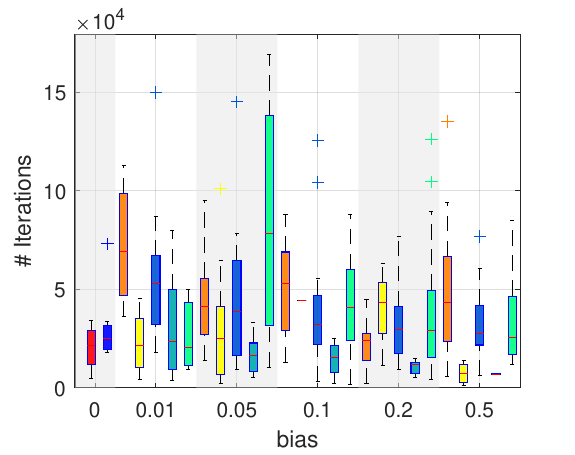}
    \end{subfigure}
    \begin{subfigure}[b]{0.32\textwidth}
        \fontsize{5}{5}\selectfont 
        \includegraphics[width=\textwidth]{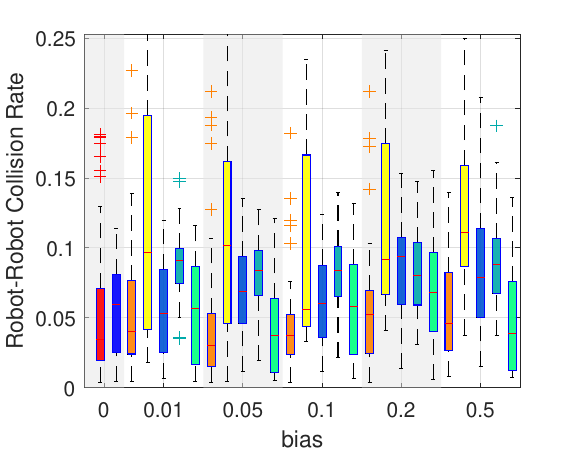}
    \end{subfigure}
    \begin{subfigure}[b]{0.32\textwidth}
        \fontsize{5}{5}\selectfont 
        \includegraphics[width=\textwidth]{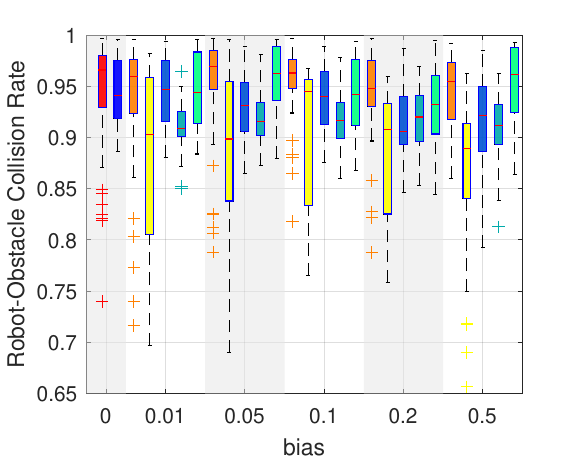}
    \end{subfigure}
    \begin{subfigure}[b]{0.32\textwidth}
        \fontsize{5}{5}\selectfont 
        \includegraphics[width=\textwidth]{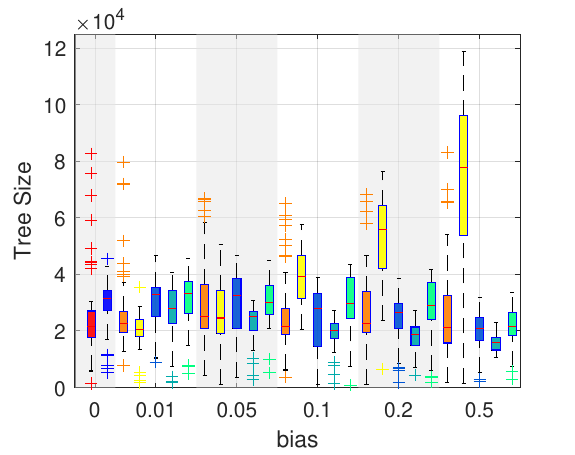}
    \end{subfigure}
    \caption{Hive 3 Results}
    \label{fig:Hive 3 Results}
\end{figure*}

\begin{figure*}
    \centering
    \begin{subfigure}[b]{0.32\textwidth}
        \centering
        \includegraphics[width=\textwidth,trim={0cm 0.7cm 0cm 0cm}, clip]{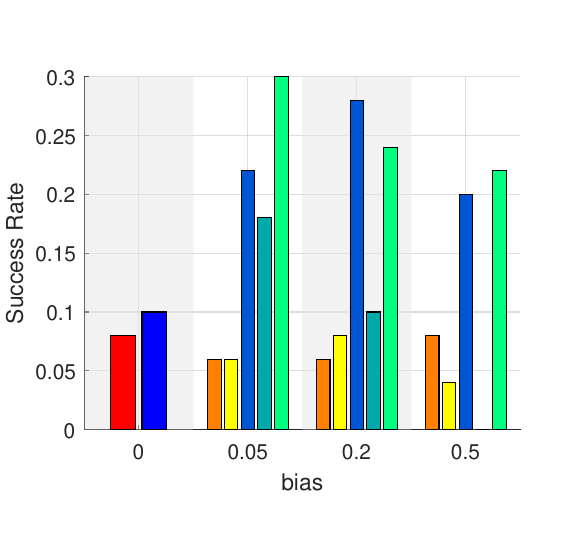}
    \end{subfigure}
    \begin{subfigure}[b]{0.32\textwidth}
        \centering
        \includegraphics[width=\textwidth]{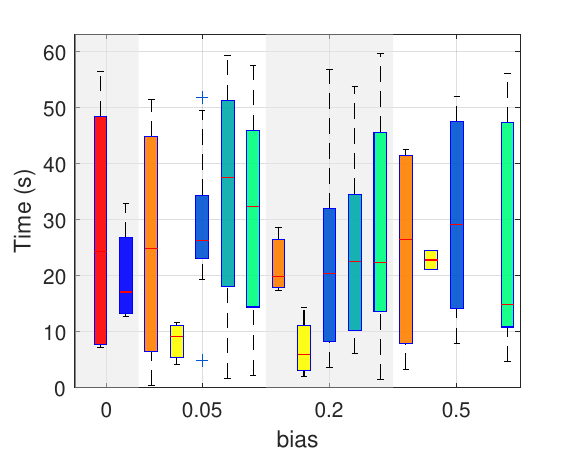}
    \end{subfigure}
    \begin{subfigure}[b]{0.32\textwidth}
        \fontsize{5}{5}\selectfont 
        \includegraphics[width=\textwidth]{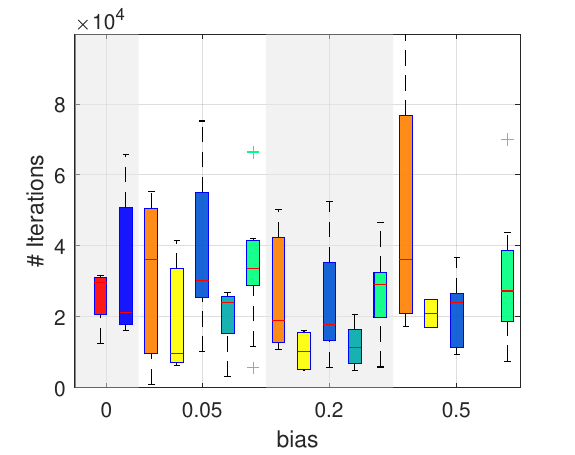}
    \end{subfigure}
    \begin{subfigure}[b]{0.32\textwidth}
        \fontsize{5}{5}\selectfont 
        \includegraphics[width=\textwidth]{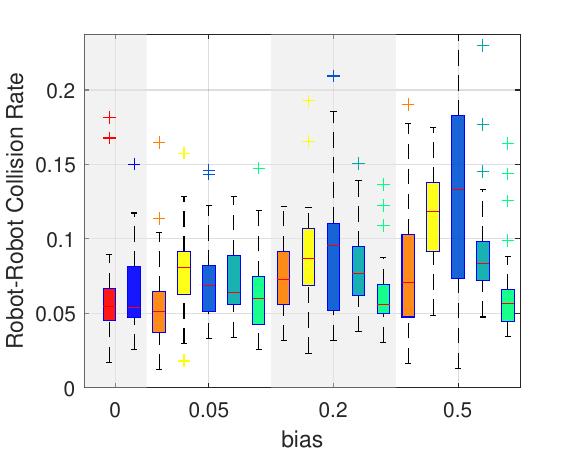}
    \end{subfigure}
    \begin{subfigure}[b]{0.32\textwidth}
        \fontsize{5}{5}\selectfont 
        \includegraphics[width=\textwidth]{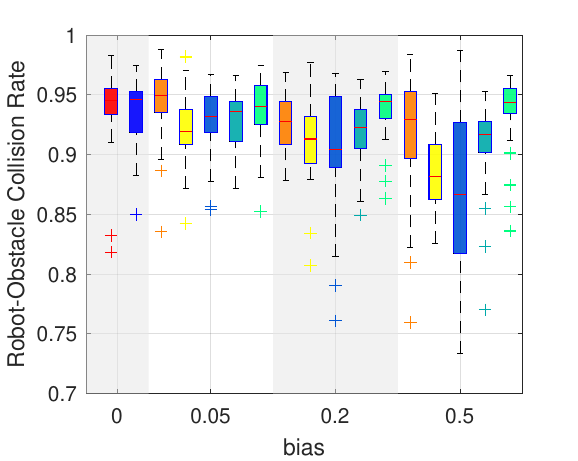}
    \end{subfigure}
    \begin{subfigure}[b]{0.32\textwidth}
        \fontsize{5}{5}\selectfont 
        \includegraphics[width=\textwidth]{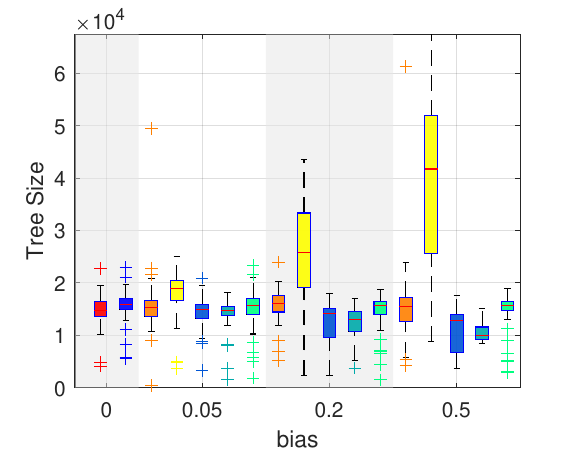}
    \end{subfigure}
    \caption{Hive 4 Results}
\label{fig:Hive 4 Results}
\end{figure*}

\begin{figure*}
    \centering
        \begin{subfigure}[b]{0.25\paperwidth}
        \centering
        \includegraphics[width=\textwidth,trim={0cm 0.7cm 0cm 0cm}, clip]{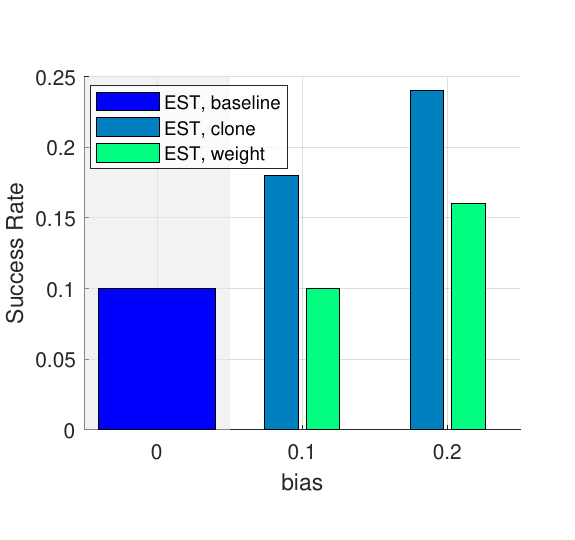}
    \end{subfigure}
    \begin{subfigure}[b]{0.3\textwidth}
        \centering
        \includegraphics[width=\textwidth]{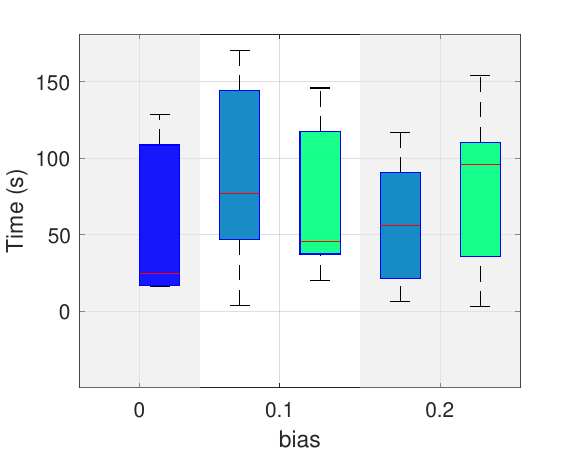}
    \end{subfigure}
    \begin{subfigure}[b]{0.25\paperwidth}
        \centering
        \includegraphics[width=\textwidth]{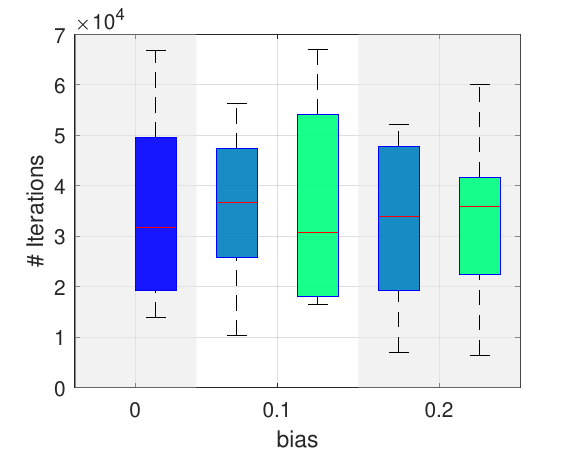}
    \end{subfigure}
    \caption{Hive 5 Results}
\label{fig:Hive5 Results}
\end{figure*}

\begin{figure*}
    \centering
    \begin{subfigure}[b]{0.25\paperwidth}
        \centering
        \includegraphics[width=\textwidth,trim={0cm 0.7cm 0cm 0cm}, clip]{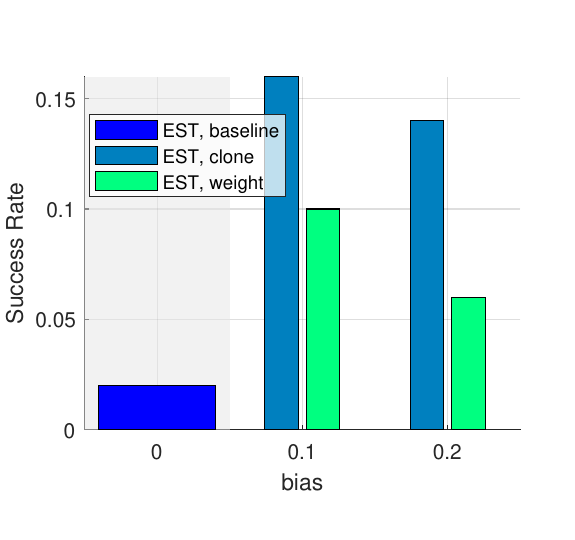}
    \end{subfigure}
    \begin{subfigure}[b]{0.3\textwidth}
        \centering
        \includegraphics[width=\textwidth]{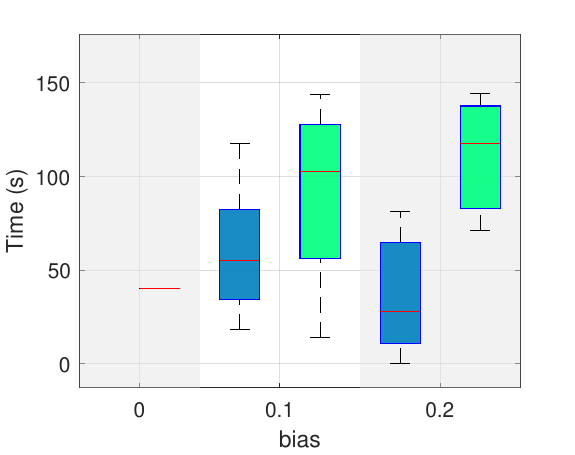}
    \end{subfigure}
    \begin{subfigure}[b]{0.25\paperwidth}
        \centering
        \includegraphics[width=\textwidth]{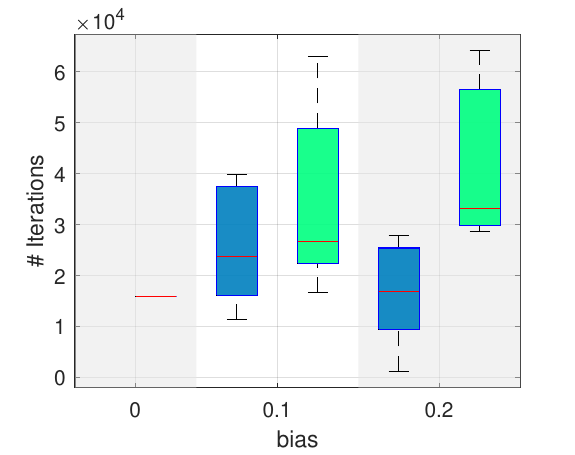}
    \end{subfigure}
    \caption{Hive 6 Results}
    \label{fig:Hive6 Results}
\end{figure*}

\clearpage
\begin{figure*}
    \centering
    \begin{subfigure}[b]{0.9\textwidth}
        \centering
        \includegraphics[width=\textwidth]{figures/legend_crop.pdf}
    \end{subfigure}
    \begin{subfigure}[b]{0.32\textwidth}
        \centering
        \fontsize{5}{5}\selectfont 
        \includegraphics[width=\textwidth,trim={0cm 0.7cm 0cm 0cm}, clip]{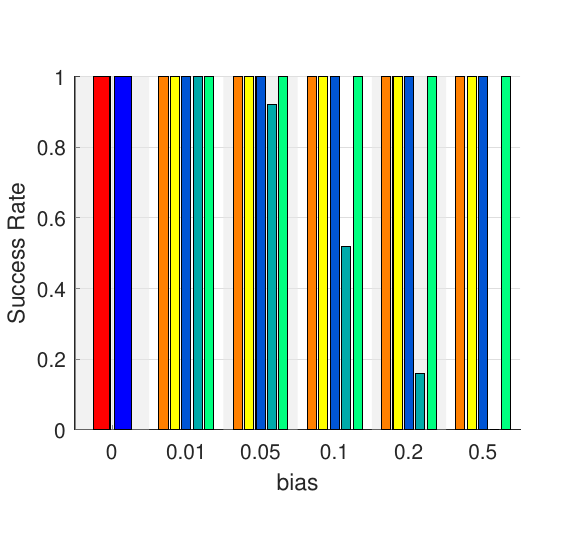}
    \end{subfigure}
    \begin{subfigure}[b]{0.32\textwidth}
        \centering
        \fontsize{5}{5}\selectfont 
        \includegraphics[width=\textwidth]{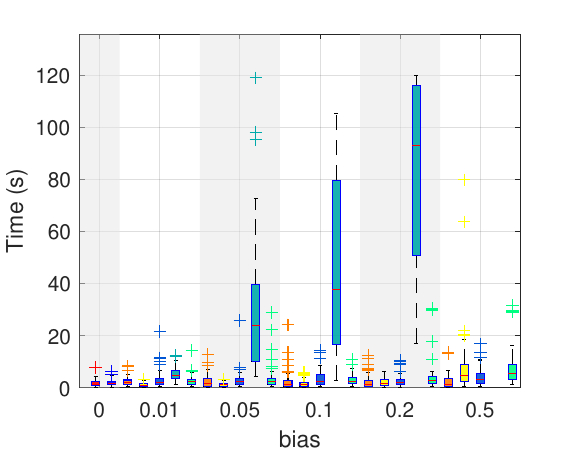}
    \end{subfigure}
    \begin{subfigure}[b]{0.32\textwidth}
        \centering
        \fontsize{5}{5}\selectfont 
        \includegraphics[width=\textwidth]{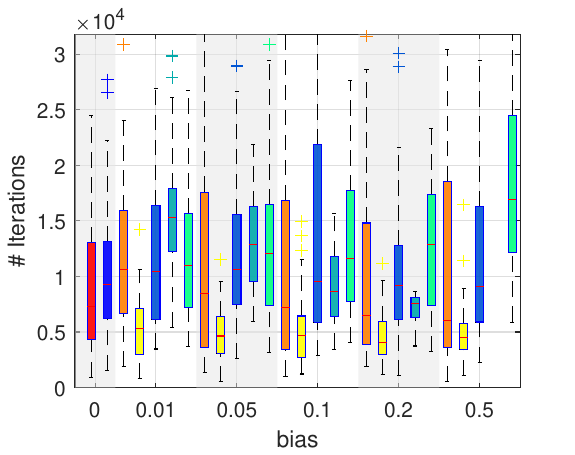}
    \end{subfigure}
    \caption{Pincer 2 Results}
\label{fig:Pincer2 Results}
\end{figure*}

\begin{figure*}
    \centering
    \begin{subfigure}[b]{0.32\textwidth}
        \centering
        \fontsize{5}{5}\selectfont 
        \includegraphics[width=\textwidth,trim={0cm 0.7cm 0cm 0cm}, clip]{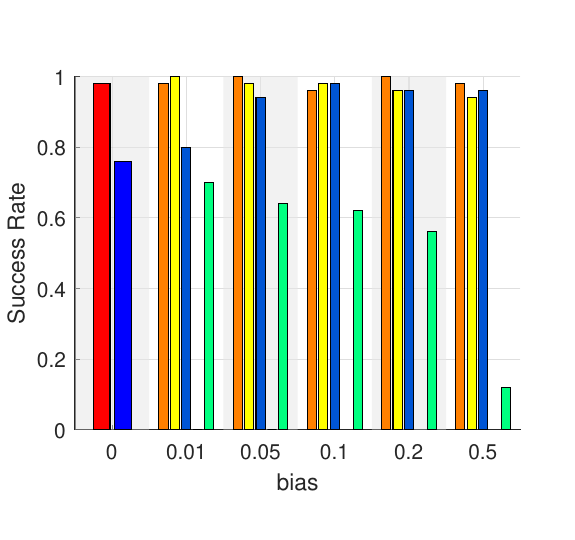}
    \end{subfigure}
    \begin{subfigure}[b]{0.32\textwidth}
        \centering
        \fontsize{5}{5}\selectfont 
        \includegraphics[width=\textwidth]{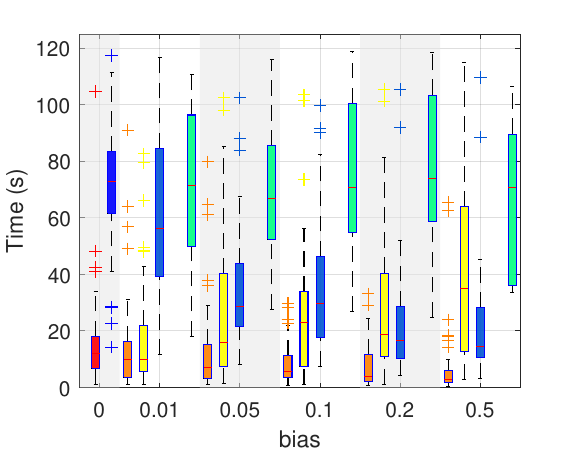}
    \end{subfigure}
    \begin{subfigure}[b]{0.32\textwidth}
        \centering
        \fontsize{5}{5}\selectfont 
        \includegraphics[width=\textwidth]{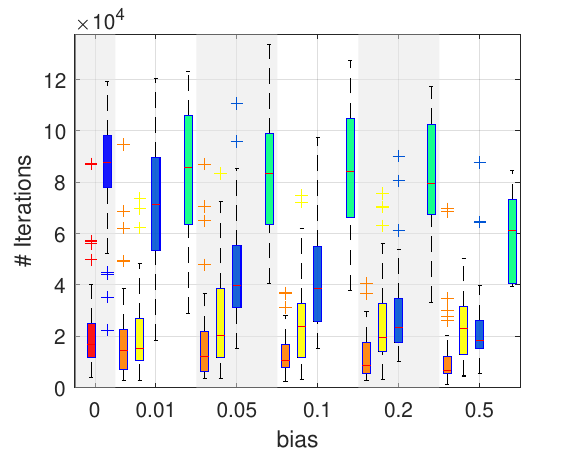}
    \end{subfigure}
    \caption{Pincer 3 Results}
\label{fig:Pincer3 Results}
\end{figure*}

\begin{figure*}
    \centering
    \begin{subfigure}[b]{0.32\textwidth}
        \centering
        \fontsize{5}{5}\selectfont 
        \includegraphics[width=\textwidth,trim={0cm 0.7cm 0cm 0cm}, clip]{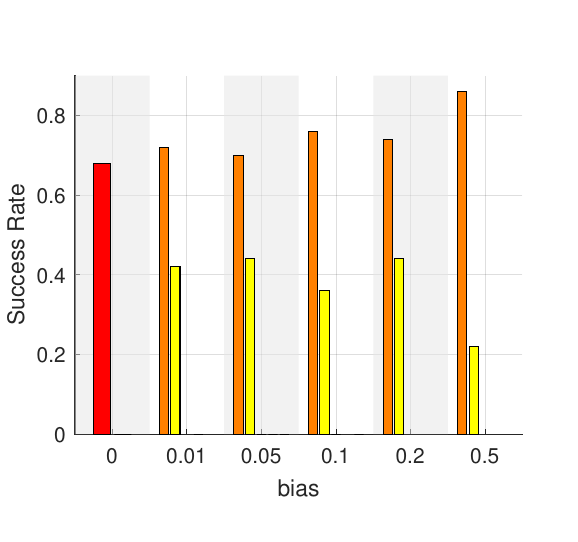}
    \end{subfigure}
    \begin{subfigure}[b]{0.32\textwidth}
        \centering
        \fontsize{5}{5}\selectfont 
        \includegraphics[width=\textwidth]{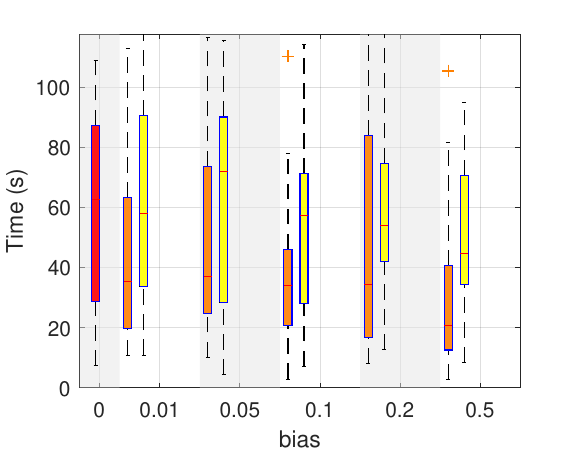}
    \end{subfigure}
    \begin{subfigure}[b]{0.32\textwidth}
        \centering
        \fontsize{5}{5}\selectfont 
        \includegraphics[width=\textwidth]{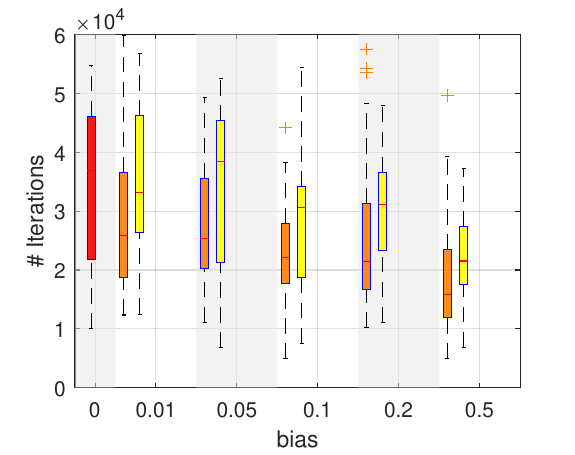}
    \end{subfigure}
    \caption{Pincer 4 Results}
    \label{fig:Pincer4 Results}
    \vspace{-0.5cm}
\end{figure*}

\begin{figure*}
    \centering
    \begin{subfigure}[b]{0.32\textwidth}
        \centering
        \fontsize{5}{5}\selectfont 
        \includegraphics[width=\textwidth,trim={0cm 0.7cm 0cm 0cm}, clip]{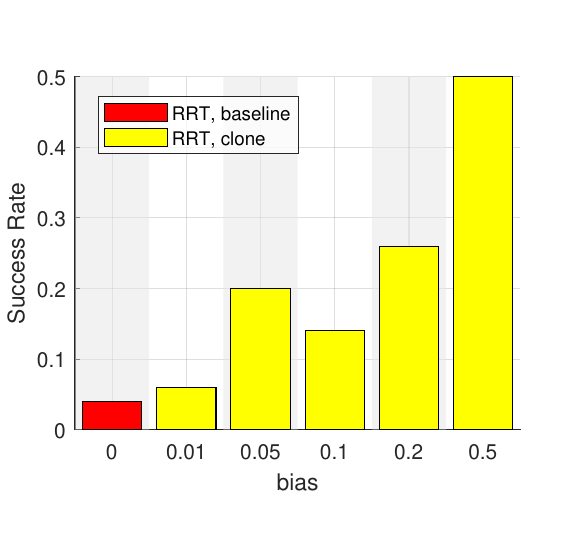}
    \end{subfigure}
    \begin{subfigure}[b]{0.32\textwidth}
        \centering
        \fontsize{5}{5}\selectfont 
        \includegraphics[width=\textwidth]{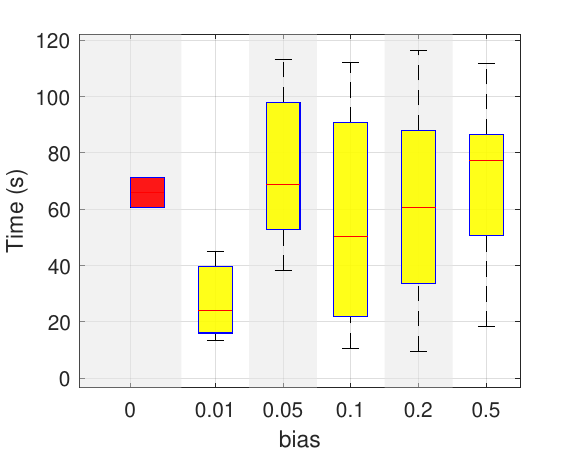}
    \end{subfigure}
    \begin{subfigure}[b]{0.32\textwidth}
        \centering
        \fontsize{5}{5}\selectfont 
        \includegraphics[width=\textwidth]{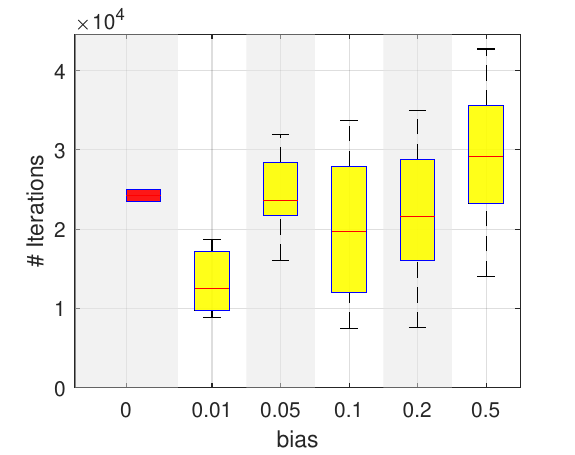}
    \end{subfigure}
    \caption{Pincer 5 Results}
    \label{fig:Pincer5 Results}
    \vspace{-0.5cm}
\end{figure*}

\end{document}